\documentclass[pdflatex,sn-nature]{sn-jnl}
\usepackage{graphicx}
\usepackage[T1]{fontenc}
\usepackage{lmodern}
\usepackage{amsmath,amssymb,amsfonts,amsthm}
\usepackage{booktabs,array}
\usepackage{placeins}
\usepackage{url}
\hypersetup{hypertexnames=false}
\newcommand{\vect}[1]{\boldsymbol{#1}}
\newcommand{\mat}[1]{\mathbf{#1}}
\newcommand{\ones}{\mathbf{1}}
\newcommand{\zeros}{\mathbf{0}}
\newcommand{\R}{\mathbb{R}}
\newcommand{\E}{\mathbb{E}}
\newcommand{\cH}{\mathcal{H}}
\newcommand{\cD}{\mathcal{D}}

\newcommand{\Var}{\operatorname{Var}}
\newcommand{\TV}{\operatorname{TV}}


\newtheorem{theorem}{Supplementary Theorem}
\newtheorem{lemma}{Supplementary Lemma}
\newtheorem{proposition}{Supplementary Proposition}
\newtheorem{corollary}{Supplementary Corollary}
\begin{document}
\title[Solution-space heterogeneity in federated PDE learning]{Solution-space heterogeneity shapes federated learning dynamics across partial differential equations}
\author[1]{\fnm{Ping} \sur{Luo}}\email{luoping@nudt.edu.cn}
\author[1]{\fnm{Jiahuan} \sur{Wang}}\email{wangjiahuan@nudt.edu.cn}
\author[1]{\fnm{Ziqing} \sur{Wen}}\email{zqwen@nudt.edu.cn}
\author*[1]{\fnm{Kele} \sur{Xu}}\email{kele.xu@ieee.org}
\equalcont{Kele Xu, Dongsheng Li and Tao Sun contributed equally to this work.}
\author*[1]{\fnm{Dongsheng} \sur{Li}}\email{dsli@nudt.edu.cn}
\equalcont{Kele Xu, Dongsheng Li and Tao Sun contributed equally to this work.}
\author*[1]{\fnm{Tao} \sur{Sun}}\email{suntao.saltfish@outlook.com}
\equalcont{Kele Xu, Dongsheng Li and Tao Sun contributed equally to this work.}
\affil[1]{\orgdiv{PDL Lab, College of Computer Science and Technology},
  \orgname{National University of Defense Technology},
  \orgaddress{\city{Changsha}, \postcode{410073}, \country{China}}}
\abstract{In natural science applications, federated learning (FL) enables collaborative training of models that approximate solutions to partial differential equations (PDEs) for repeated use in distributed settings. However, the input conditions of FL-PDE training samples involve multiple coupled parameters, making data heterogeneity across clients difficult to analyse. We propose solution-space PDE-Dirichlet, which groups existing solution labels into task-specific bins and distributes samples using the classical Dirichlet allocation rule in FL, providing a common partitioning approach across PDE tasks. Theoretical analysis shows an association between differences in solution-bin composition and gradient disagreement, while experiments on three controlled operators and four public simulation datasets quantify the relationships between this heterogeneity, gradient disagreement and prediction error. Based on this connection, we develop Solution-Bin Controlled Federated Optimization, which uses pooled bin-gradient prototypes to correct the component of client gradients associated with bin composition while preserving the sample-weighted training objective. Experiments on low-viscosity Burgers and Darcy flow show that this correction reduces prediction error compared with federated averaging and narrows the gap from bin-balanced IID training.}
\keywords{federated learning, partial differential equations, neural operators, solution-space heterogeneity, optimal transport}
\maketitle
\section{Introduction}
Partial differential equations (PDEs) describe transport, diffusion, reaction and wave propagation in fluids, materials and environmental systems, where prediction, inverse inference and design often require repeated solutions under different physical conditions \cite{karniadakis2021physics,azizzadenesheli2024operators}. Neural operators learn maps from input functions to solution fields so that trained models can approximate new responses without repeating the full numerical solution procedure \cite{lu2021deeponet,li2021fno,kovachki2023neuraloperator}. When the training data are held by different institutions, federated learning (FL) provides a framework for training these models collaboratively through local updates and model aggregation while keeping training examples with their owners \cite{mcmahan2017fedavg,kairouz2021advances}. This arrangement supports scientific collaboration when measurements or simulation outputs are sensitive or subject to data-sharing restrictions.

Fed-DeepONet demonstrated federated training of deep operator networks (DeepONets), and subsequent federated scientific machine learning examined function approximation and differential equations with heterogeneous client data \cite{moya2022feddeeponet,zhang2024fedsciml}. In these settings, each training example can contain coefficient fields, forcing terms, initial or boundary conditions, and geometric information whose combined influence determines the response. Clients that sample different physical conditions therefore contribute different training populations, but the coupling between these inputs makes it difficult to describe their heterogeneity through any one parameter. Understanding how such differences affect gradients and prediction requires a partitioning approach that can be applied consistently across PDE tasks.

Class-label allocation provides a common way to construct non-independent and identically distributed (non-IID) data in FL, with Dirichlet concentration controlling how unevenly examples from each class are distributed across clients \cite{hsu2019noniid,li2022niid}. Applying an analogous construction to PDE inputs requires choosing which variables to partition and how to discretize them, because a scalar coefficient, a boundary function and a spatially varying medium have different representations. Input-based partitions can isolate particular physical factors, but the rules needed to combine these factors vary across equations. The response fields already available as training labels offer an alternative basis: they express the outcome of the governing dynamics under the supplied conditions and are the quantities that the predictive model must learn.

We introduce solution-space PDE-Dirichlet, which groups normalized training response fields into task-specific solution bins and distributes the examples in each bin among clients through a common Dirichlet allocation rule (Fig.~\ref{fig:system}). Clustering-based allocation has also been used in non-classification benchmarks such as FedNLP \cite{lin2022fednlp}; here, the bins are defined by the physical fields that the model is trained to predict. Each task retains its own response representation and bin centres, while normalization, clustering and allocation follow the same procedure across tasks. The resulting partitions vary both solution-bin composition and client sample size, and transport between client bin distributions measures their realized separation. For inverse problems, the bins describe the target physical field to be recovered.

\begin{figure}[!htbp]
\centering
\includegraphics[width=.95\linewidth,height=.48\textheight,keepaspectratio]{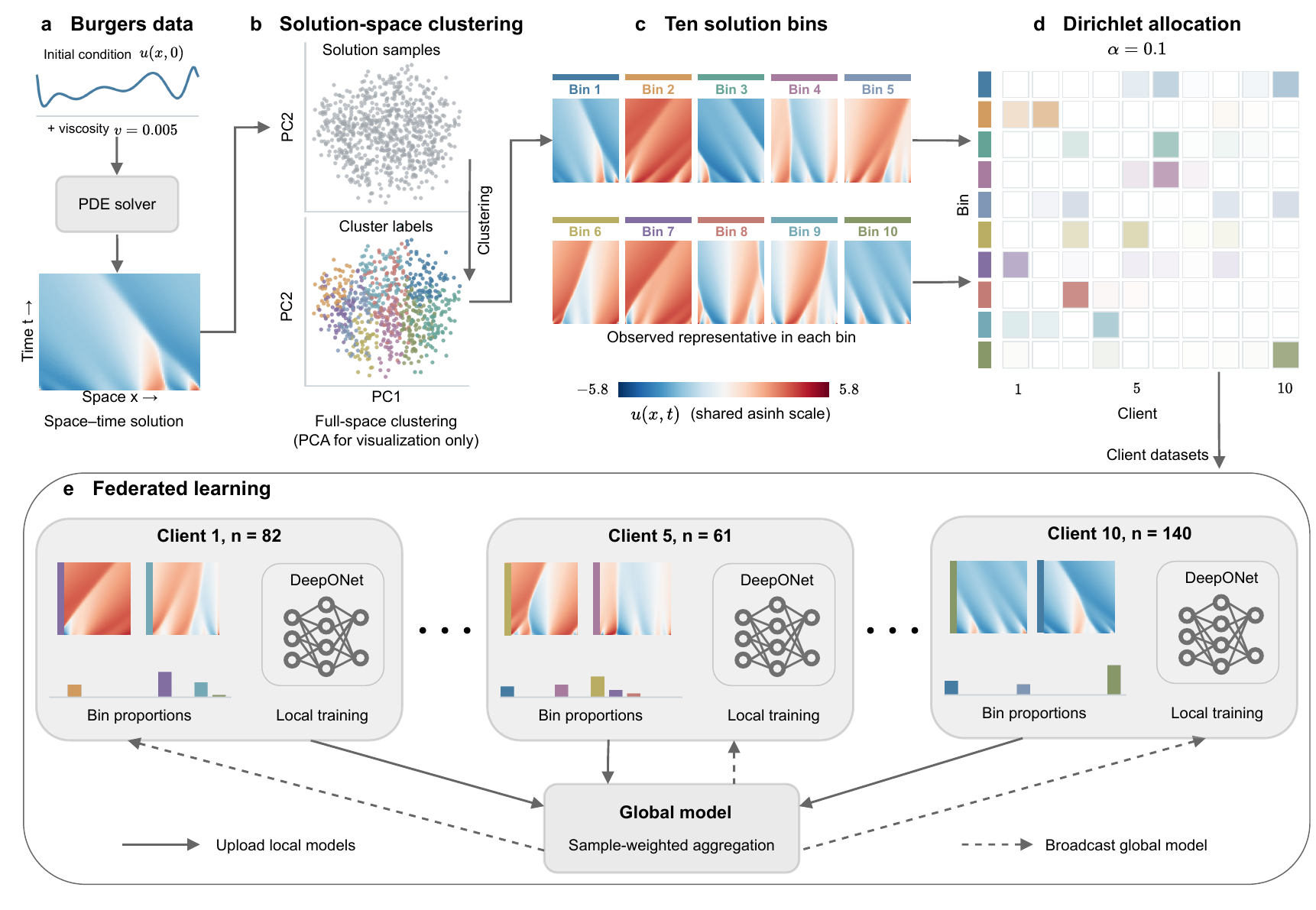}
\caption{Solution-space partitioning and federated learning, illustrated with Burgers data. \textbf{a}, Initial conditions $u(x,0)$ and viscosity $\nu=0.005$ define numerical space--time solutions. \textbf{b}, The same solutions are projected onto two principal components before and after colouring by cluster identity; clustering uses the full normalized response vectors. \textbf{c}, Ten bins are represented by observed training solutions nearest their respective centroids. Bin colours are consistent across panels, and solution fields share an inverse-hyperbolic-sine colour scale. \textbf{d}, At $\alpha=0.1$, each row shows the fraction of one bin allocated across ten clients. Samples are assigned without replacement. \textbf{e}, Clients 1, 5 and 10 illustrate local datasets with 82, 61 and 140 samples, respectively. Their histograms show within-client bin proportions. Local models are uploaded for sample-weighted aggregation, and the global model is broadcast for the next round. The fields and counts illustrate one 800-sample training pool; network icons are schematic.}\label{fig:system}
\end{figure}

We analyse how the allocation changes client bin proportions and how these composition differences contribute to gradient disagreement. The analysis separates the contribution of bin composition from within-bin variation and establishes conditional bounds on local update differences that account for persistent optimizer states. Experiments across three controlled operators and four public simulation datasets quantify the relationships between solution-bin heterogeneity, gradient disagreement and prediction error. The controlled tasks comprise antiderivative, diffusion--reaction and Burgers operators, while the public tasks cover Darcy flow, shallow-water dynamics, cylinder flow and seismic inversion \cite{takamoto2022pdebench,luo2023cfdbench,deng2022openfwi}. For each task, paired comparisons with a size-balanced, bin-stratified IID reference measure the joint effect of composition and sample-size differences under matched training settings \cite{wang2021fieldguide}.

This gradient decomposition also provides a basis for reducing the effect of bin-composition differences during training. We develop Solution-Bin Controlled Federated Optimization (SB-CFO), which uses pooled bin-gradient prototypes and the difference between local and global bin proportions to construct a correction for each client. Related control-variate methods, including SCAFFOLD, correct local updates to address client drift \cite{karimireddy2020scaffold}; SB-CFO uses the solution bins to identify the composition-related component of the gradient difference. With fresh exact prototypes evaluated at the broadcast model, the correction cancels this component, and its zero sample-weighted mean preserves the pooled-data training objective. Experiments on low-viscosity Burgers and Darcy flow show that SB-CFO reduces prediction error relative to federated averaging (FedAvg) and narrows the gap from bin-balanced IID training when communication rounds and local optimizer steps are matched.
\FloatBarrier

\section{Results}
\subsection{A shared construction defines task-specific client distributions}
The benchmark covers three controlled operators and four public simulation tasks (Table~\ref{tab:tasks}). Each training set is partitioned into ten response bins. For every bin, a symmetric Dirichlet draw with concentration $\alpha$ determines its allocation across ten clients. Decreasing $\alpha$ favours concentrated allocations. The reference distributes each bin as evenly as integer counts permit, while keeping client sizes within one sample of one another. We call this the stratified IID reference. It preserves the global bin proportions, which need not be uniform, and is constructed independently of $\alpha$.

Figure~\ref{fig:bins} shows one fixed Burgers partition. The IID clients contain similar mixtures, whereas smaller concentrations produce more concentrated mixtures and unequal client sizes. All conditions partition the same pooled training set. Transport between the resulting bin histograms measures realized separation, denoted $D_{\mathrm{sol}}$. Across tasks, its magnitude depends on the data geometry as well as on $\alpha$ (Supplementary Fig.~5). For example, at $\alpha=0.01$, the mean distance is 0.95 for the controlled Burgers task and 1.90 for Darcy flow. 

\begin{table}[!htbp]
\caption{Tasks and evaluated dataset sizes. Counts refer to supervised examples, except that the cylinder split is first made by physical case. The antiderivative is a simple operator-learning control. All public data are numerical simulations. FNO, Fourier neural operator.}\label{tab:tasks}
\centering\small
\begin{tabular}{p{.22\linewidth}p{.32\linewidth}rrp{.15\linewidth}}
\toprule
Task & Prediction target & Training & Test & Model\\
\midrule
Antiderivative & Integrated input function & 1,000 & 1,000 & DeepONet\\
Diffusion--reaction & Space--time solution & 1,000 & 1,000 & DeepONet\\
Burgers & Space--time solution, $\nu=0.1$ & 800 & 500 & DeepONet\\
Darcy & Pressure, $\beta=0.01$ data & 2,000 & 500 & FNO\\
Shallow water & Height at four future times & 800 & 200 & FNO\\
Cylinder flow & Two velocity components at three horizons & 2,784 & 300 & FNO\\
CurveFault-B & Subsurface velocity field & 2,000 & 500 & InversionNet-lite\\
\bottomrule
\end{tabular}
\end{table}

\begin{figure}[!htbp]
\centering\includegraphics[width=.97\linewidth]{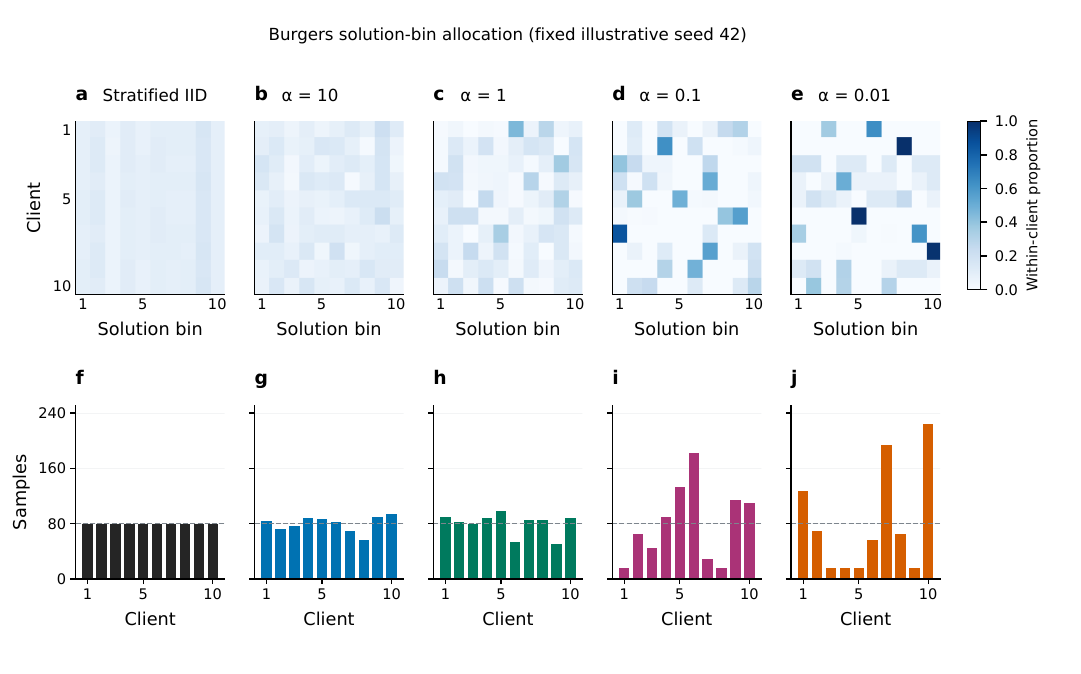}
\caption{Response-bin mixtures and sample counts under the partition protocol. Burgers partitions use the fixed illustrative seed 42, ten clients and ten bins. \textbf{a--e}, Within-client bin proportions for stratified IID and $\alpha=10,1,0.1,0.01$, with a common colour scale. \textbf{f--j}, Corresponding client sample counts; the dashed line denotes the mean count. Variability across five seeds is shown in Supplementary Fig.~6.}\label{fig:bins}
\end{figure}

\subsection{Allocation imbalance has conditional effects on updates}
Let $N_b$ be the number of examples in bin $b$, $N=\sum_bN_b$ and $\pi_b=N_b/N$. With $K$ clients and independent allocations $\mat{R}_{\cdot b}\sim\mathrm{Dirichlet}(\alpha\ones_K)$, define the raw mass profile $u_{kb}=K\pi_bR_{kb}$. Its client mean is $\pi_b$. The population heterogeneity satisfies
\begin{equation}
\E\cH_u=\frac{K-1}{K\alpha+1}\sum_b\pi_b^2,
\quad \cH_u=\frac1K\sum_k\|\vect{u}_k-\vect{\pi}\|_2^2.
\label{eq:population}
\end{equation}
For the implemented multinomial allocation, let $m_{kb}$ be the integer counts before minimum-size repair and $\omega_{kb}=Km_{kb}/N$. Supplementary Theorem~1 gives the finite-sample counterpart
\begin{equation}
\E\cH_m=\frac{K-1}{K\alpha+1}
\left(\sum_b\pi_b^2+\frac{K\alpha}{N}\right),
\quad \cH_m=\frac1K\sum_k\|\vect{\omega}_k-\vect{\pi}\|_2^2.
\label{eq:finite}
\end{equation}
The additional term captures allocation noise. Both expectations decrease with $\alpha$, with strict decrease in the finite-sample expression whenever some bin contains more than one example. The latter approaches $(K-1)/N$ as $\alpha\to\infty$. The stratified reference controls this finite-sample imbalance directly.

Normalized client histograms $\vect{q}_k$ differ from the raw profiles because client masses vary. Supplementary Lemma~1 quantifies this normalization, and Supplementary Note~3 bounds the histogram imbalance of the finite stratified reference. Minimum-size repair further changes the allocation law. The experiments therefore measure distances and sample counts after repair.

The link to optimization depends on how bins differ in loss. In a population mixture, write $F_k=\sum_bq_{kb}F_b^{\mathrm{bin}}$, $\vect{\mu}_b=\nabla F_b^{\mathrm{bin}}$ and $\mat{G}=[\vect{\mu}_1,\ldots,\vect{\mu}_B]$, with all gradients evaluated at a common model. Gradient heterogeneity is then
\begin{equation}
\cH_g=\sum_kp_k\|\mat{G}(\vect{q}_k-\vect{\pi})\|_2^2,
\quad p_k=N_k/N.
\label{eq:gradient_transfer}
\end{equation}
If $\mat{G}$ is bounded below on zero-sum bin contrasts, histogram imbalance gives a positive lower bound on $\cH_g$. For finite data, within-bin empirical differences add a residual whose effect is bounded explicitly (Supplementary Proposition~2).

For local Adam, the relevant gradient is filtered by moment history and an adaptive diagonal preconditioner. At a fixed round, let $\mat{D}$ be a common reference preconditioner, $\vect{a}_k$ the client gradient including any bin correction, and $\vect{m}_{k,0}$ its retained first-moment state. Define the local endpoint dispersion as $D_{\mathrm{loc}}=\sum_kp_k\|\vect{\theta}_{k,E}-\sum_jp_j\vect{\theta}_{j,E}\|_2^2$, the sample-weighted variance of client models after that round's local steps. Supplementary Theorem~2 gives
\begin{equation}
\left|\sqrt{\E D_{\mathrm{loc}}}-\eta B_E\sqrt{\cH_{Da}}\right|
\leq\eta A_E\sqrt{\cH_{Dm}}+\mathcal R,
\label{eq:local}
\end{equation}
where expectations condition on the round-start data, partition, model, prototypes and optimizer states. The quantities $\cH_{Da}$ and $\cH_{Dm}$ are sample-weighted variances after applying $\mat{D}$, and $A_E,B_E$ are the exact bias-corrected momentum weights over $E$ local steps. The remainder $\mathcal R$ bounds preconditioner variation, local gradient evolution and stochastic noise. The result assumes local smoothness, bounded corrected stochastic gradients, conditionally unbiased gradient noise and bounded deviation from the reference preconditioner.

The bin-imbalance bounds transfer to the gradient component when bin-conditioned gradients are distinguishable. A positive dispersion lower bound requires this component to exceed the moment-history and remainder terms. Fresh exact SB-CFO prototypes cancel the between-bin component at the broadcast model, leaving within-bin variation and state-dependent terms. Supplementary Notes~8 and 9 relate these terms to aggregation, finite-horizon stationarity and prediction sensitivity. This conditional treatment follows Adam analyses that account for effective steps and optimizer states \cite{reddi2018adam,chen2019adamtype,defossez2022simple}.

\subsection{Optimization shifts need not produce prediction penalties}
We compare each non-IID run with same-seed FedAvg on the stratified reference at round 1,000. The paired excess error is $\Delta_{\mathrm{err}}=100(e_\alpha-e_{\mathrm I})$, measured in percentage points (pp), where $e$ is relative $L_2$ test error. Positive values indicate worse prediction than the federated reference.

The absolute test-error curves largely overlap in the controlled tasks, but paired differences reveal their late-training behaviour (Fig.~\ref{fig:formalcurves}). At $\alpha=0.01$, the mean excess error is 0.097 pp for the antiderivative, 0.180 pp for diffusion--reaction and 0.741 pp for Burgers (Table~\ref{tab:final}). The first two pointwise intervals cross zero; the Burgers interval is [0.239, 1.243] pp. Thus, a clear change in the partition produces a measurable predictive penalty for Burgers, but an uncertain final effect for the two smoother controls. This distinction is consistent with FedAvg analyses in which gradient dissimilarity alone does not determine convergence \cite{wang2024effectiveness}.

Gradient disagreement and parameter deviation provide a different view (Fig.~\ref{fig:formalmechanisms}). From $\alpha=10$ to 0.01, mean normalized initial gradient variance rises from 0.139 to 0.528 for the antiderivative, from 0.115 to 0.619 for diffusion--reaction and from 0.095 to 0.488 for Burgers. Mean final parameter deviation also increases between these endpoints. The seed-level observations show that similar geometric shifts can have different predictive effects. Supplementary Figs.~1 and 2 show checkpoint differences and seed-level final effects.

\begin{figure}[!htbp]
\centering\includegraphics[width=.97\linewidth,height=.72\textheight,keepaspectratio]{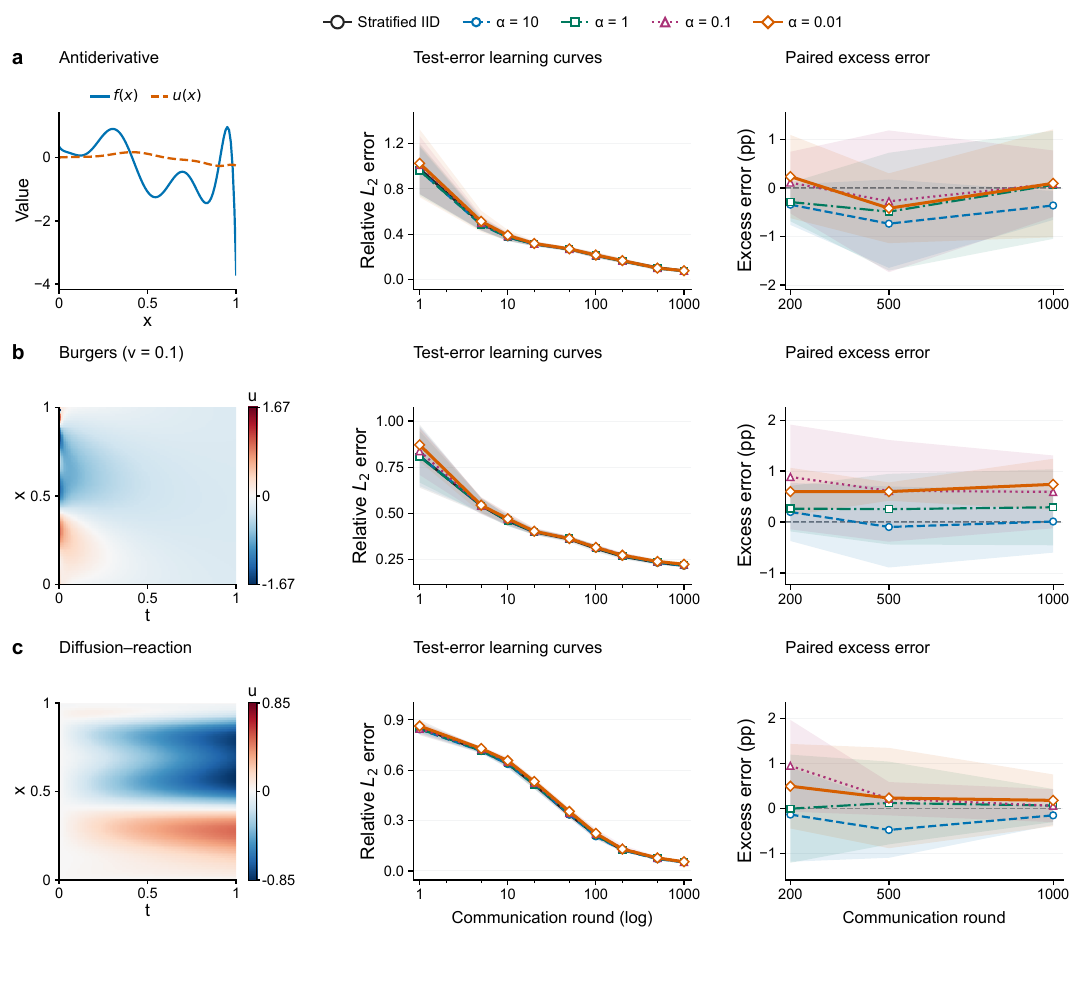}
\caption{Controlled-task examples, learning curves and paired late-training differences. Rows show antiderivative (\textbf{a}), Burgers with $\nu=0.1$ (\textbf{b}) and diffusion--reaction (\textbf{c}). Left, The first training sample from experiment seed 42: input $f(x)$ and integral $u(x)$ for antiderivative, and numerical reference solutions $u(x,t)$ for the two PDEs. Heatmaps use separate colour scales. Middle, Relative $L_2$ test error for stratified IID and four concentrations; the round axis is logarithmic. Right, Paired excess error at rounds 200, 500 and 1,000 on a linear round axis. The dashed line marks zero excess error. Means and shaded, pointwise 95\% Student-$t$ confidence intervals use five seeds. Differences are paired within seed; lines join recorded checkpoints.}\label{fig:formalcurves}
\end{figure}

\begin{figure}[!htbp]
\centering\includegraphics[width=.97\linewidth,height=.72\textheight,keepaspectratio]{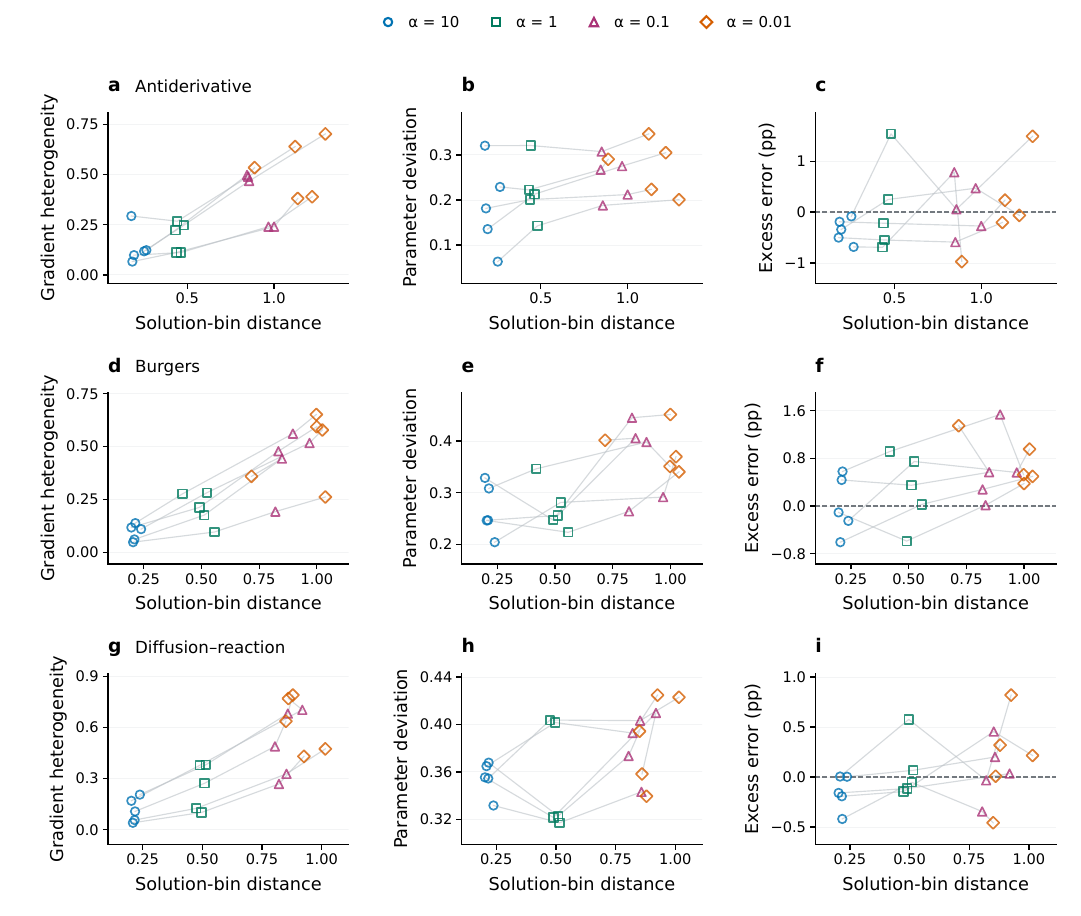}
\caption{Geometry, gradients and parameter separation in the controlled tasks. Rows show antiderivative (\textbf{a--c}), Burgers (\textbf{d--f}) and diffusion--reaction (\textbf{g--i}). Columns show normalized gradient variance at common initialization, relative parameter deviation from the same-seed IID model at round 1,000, and paired final excess error, respectively, against realized solution-bin transport distance. Each point represents one of five seeds at one of four concentrations; grey lines connect conditions within a seed.}\label{fig:formalmechanisms}
\end{figure}

\begin{table}[!htbp]
\caption{Final prediction errors and paired effects at $\alpha=0.01$. Errors are $100$ times the relative $L_2$ error, expressed as percentages. Excess error is the non-IID minus IID difference in percentage points, with a pointwise 95\% Student-$t$ confidence interval across five paired seeds. All runs use ten clients and round 1,000.}\label{tab:final}
\centering\small
\begin{tabular}{lrrl}
\toprule
Task & IID error (\%) & Non-IID error (\%) & Excess error [95\% CI] (pp)\\
\midrule
Antiderivative & 7.716 & 7.813 & 0.097 [-1.017, 1.211] \\
Burgers ($\nu=0.1$) & 21.588 & 22.328 & 0.741 [0.239, 1.243] \\
Diffusion--reaction & 5.278 & 5.459 & 0.180 [-0.397, 0.758] \\
Darcy ($\beta=0.01$) & 54.621 & 64.521 & 9.900 [5.161, 14.639] \\
Shallow water & 0.151 & 0.163 & 0.012 [0.006, 0.018] \\
Cylinder flow & 6.878 & 7.492 & 0.614 [0.343, 0.885] \\
CurveFault-B & 17.610 & 17.357 & -0.253 [-0.528, 0.022] \\
\bottomrule
\end{tabular}

\end{table}

\subsection{Physical regime and training configuration change the effect}
Reducing Burgers viscosity produces sharper numerical response fields under the fixed discretization (Fig.~\ref{fig:viscosity}). At $\alpha=0.01$, excess error increases from 0.741 pp at $\nu=0.1$ to 1.681 pp at 0.05 and 2.550 pp at 0.01. At $\nu=0.005$, it is 2.470 pp, with a 95\% interval of [1.551, 3.390] pp. The two lowest viscosities have similar mean effects. Absolute IID error rises from 21.59\% to 47.07\% across the viscosity range. Within this numerical response family, viscosity changes both baseline prediction difficulty and the additional effect of heterogeneity; Supplementary Note~10 quantifies sensitivity to grid refinement.

\begin{figure}[!htbp]
\centering\includegraphics[width=.95\linewidth]{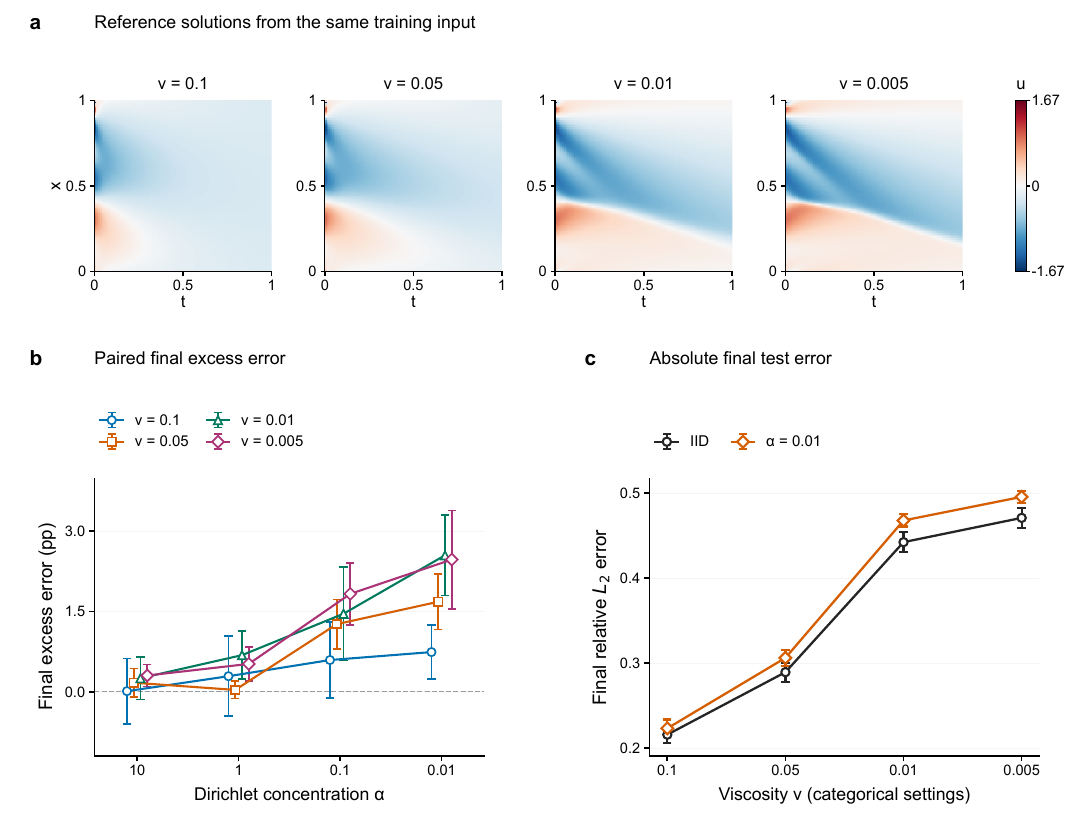}
\caption{Burgers viscosity changes solution structure, baseline difficulty and excess error. \textbf{a}, Numerical reference solutions for the same initial function at $\nu=0.1,0.05,0.01,0.005$, using the first training sample from experiment seed 42 and a shared colour scale. \textbf{b}, Paired final excess error for four concentrations. \textbf{c}, Absolute final relative $L_2$ errors for stratified IID and $\alpha=0.01$. Means and error bars show pointwise 95\% Student-$t$ intervals across five seeds. The $\nu=0.1$ results reuse the principal Burgers experiment.}\label{fig:viscosity}
\end{figure}

Increasing the client count from ten to twenty does not uniformly enlarge the final error gap (Supplementary Fig.~3). At $\alpha=0.01$, the twenty-client effects are $-0.057$ pp for the antiderivative, 0.218 pp for diffusion--reaction and 0.453 pp for Burgers; all three intervals cross zero. Each is evaluated against its own twenty-client IID run. With fixed local steps and full participation, this comparison changes the aggregate training workload as well as client sample sizes.

The optimizer ablation uses Nesterov stochastic gradient descent (SGD) with momentum 0.9 and cosine decay from 0.01 to 0.001, compared with Adam decaying from 0.001 to 0.0001 (Supplementary Fig.~4). SGD has higher IID errors: 13.60\%, 17.64\% and 26.28\% for the antiderivative, diffusion--reaction and Burgers, respectively, versus 7.72\%, 5.28\% and 21.59\% for Adam. At $\alpha=0.01$, its mean paired penalties are also larger: 0.408, 0.730 and 2.953 pp, respectively. All three SGD intervals cross zero; the Burgers interval spans [$-0.288$, 6.194] pp. Partition sensitivity therefore remains uncertain across seeds despite the larger SGD means. This comparison combines optimizer and schedule changes at a fixed step budget.

\subsection{Public tasks show distinct prediction responses}
The public evaluation spans Darcy pressure prediction, multihorizon shallow-water dynamics, conditioned cylinder-flow prediction and seismic inversion on CurveFault-B \cite{takamoto2022pdebench,luo2023cfdbench,deng2022openfwi}. These tasks extend the evaluation to FNO and a compact inversion network. The same response-bin construction applies to all four, including the velocity field recovered by the inverse problem.

\begin{figure}[!htbp]
\centering\includegraphics[width=.97\linewidth,height=.74\textheight,keepaspectratio]{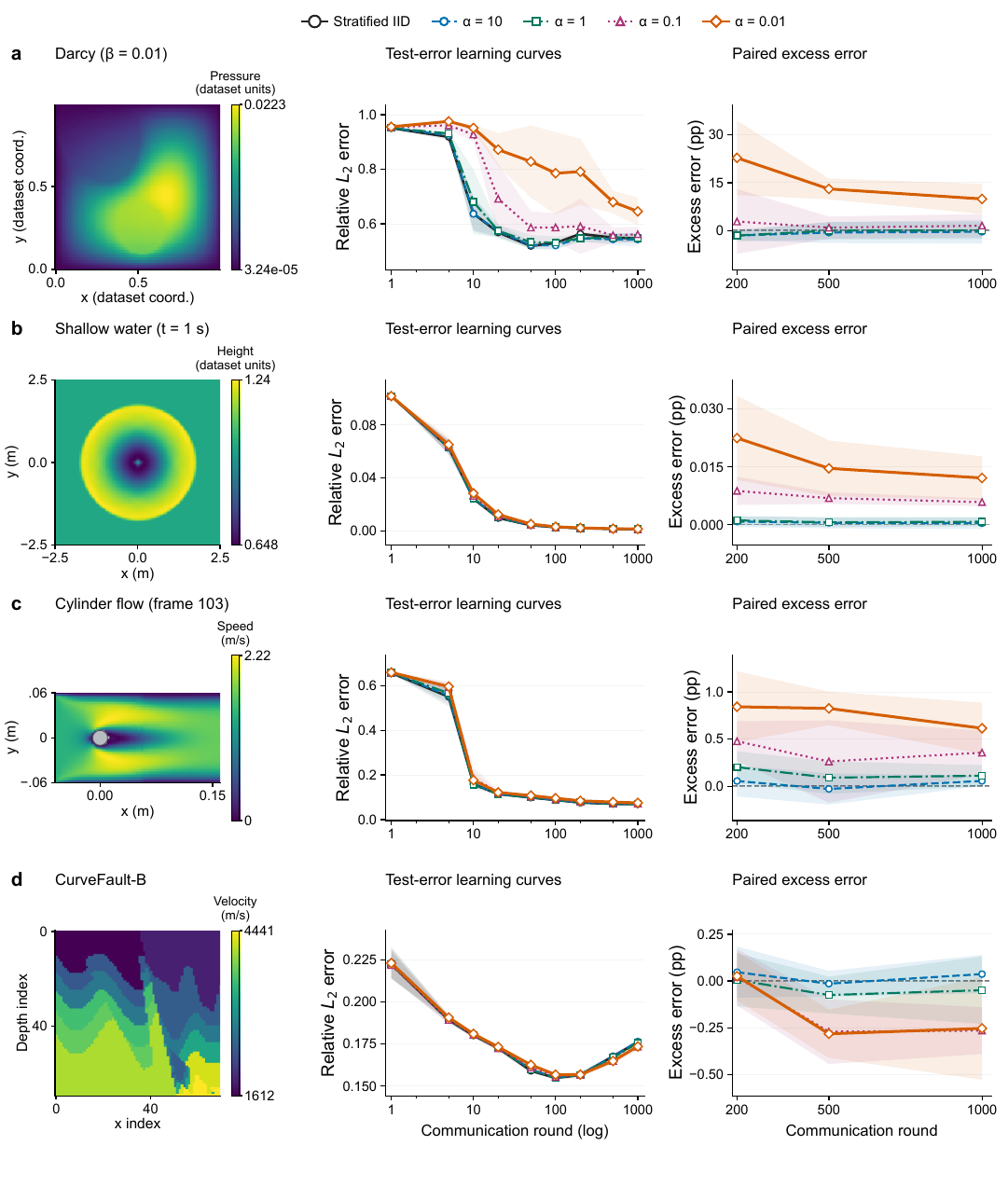}
\caption{Public-task examples, learning curves and paired late-training differences. Rows show Darcy (\textbf{a}), shallow water (\textbf{b}), cylinder flow (\textbf{c}) and CurveFault-B (\textbf{d}). Left, Reference fields from the first sample in each fixed training set: Darcy pressure, shallow-water height, cylinder-flow speed and subsurface velocity, with separate colour scales. Pressure and height retain dataset units; speed and velocity are in $\mathrm{m\,s^{-1}}$. Shallow water shows target index 100 ($t=1\,\mathrm{s}$); cylinder flow shows frame 103, 100 frames after the last input frame. Indices are zero-based; grey marks the cylinder. Middle, Relative $L_2$ test error through round 1,000 on a logarithmic round axis. Right, Within-seed non-IID minus IID differences at rounds 200, 500 and 1,000, expressed in percentage points. Shaded regions are pointwise 95\% Student-$t$ intervals across five training and partition seeds, conditional on the fixed public-data split and response codebook. Each task uses its own error scale.}\label{fig:publiccurves}
\end{figure}

Darcy flow with the $\beta=0.01$ dataset has the largest final paired effect (Fig.~\ref{fig:publiccurves}; Table~\ref{tab:final}). At $\alpha=0.01$, mean relative error increases from 54.62\% under IID to 64.52\%, an excess of 9.90 pp with a 95\% interval of [5.16, 14.64] pp. In shallow-water prediction, the corresponding excess is 0.0121 pp (95\% CI, 0.0065--0.0178), from an IID error of 0.1511\%. Cylinder flow has a positive mean excess of 0.614 pp (95\% CI, 0.343--0.885), from an IID error of 6.878\%. The magnitude of the penalty therefore differs substantially even among tasks with positive pointwise intervals.

\begin{figure}[!htbp]
\centering\includegraphics[width=.97\linewidth,height=.74\textheight,keepaspectratio]{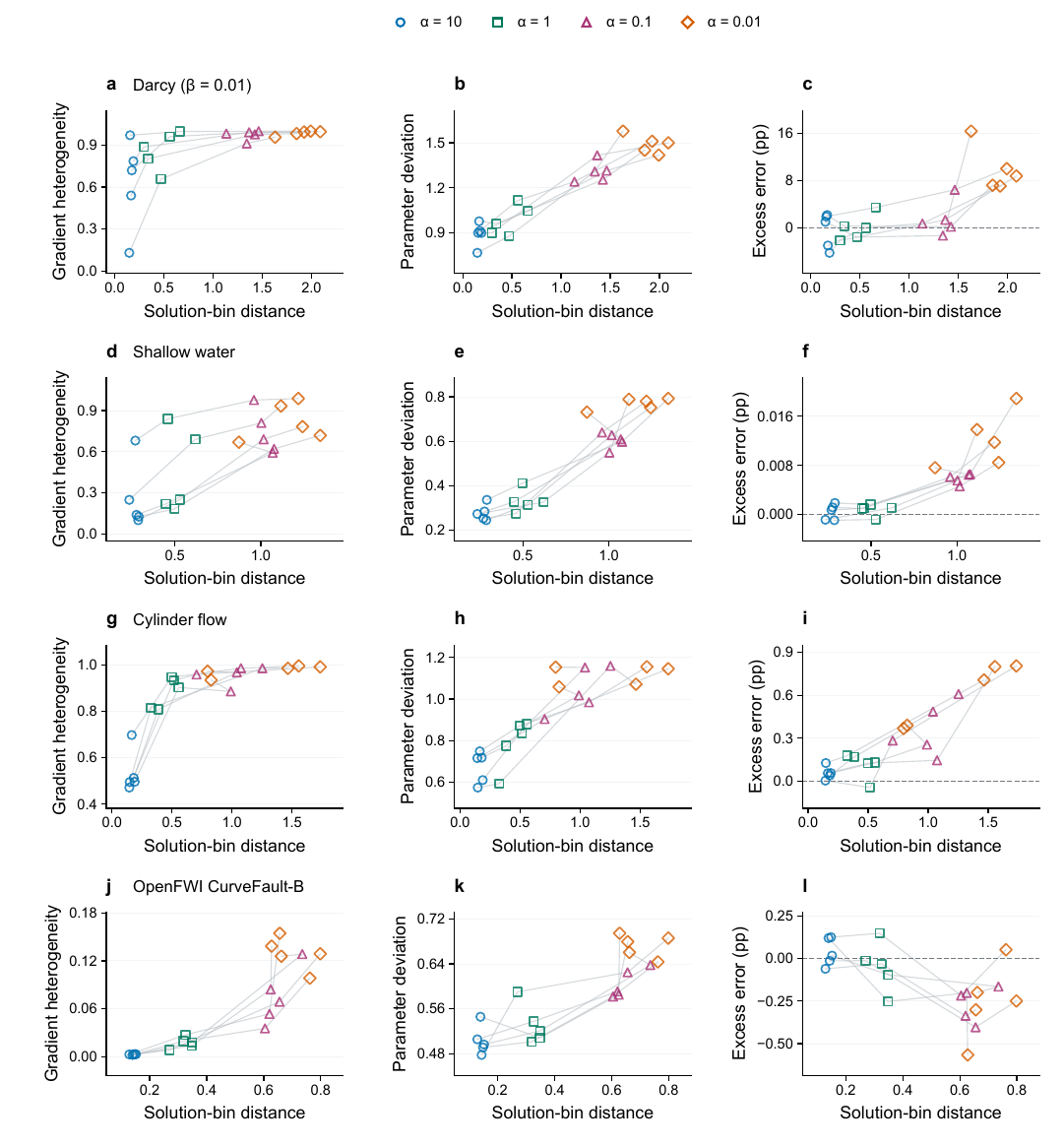}
\caption{Response separation and optimization diagnostics in public tasks. Rows show Darcy (\textbf{a--c}), shallow water (\textbf{d--f}), cylinder flow (\textbf{g--i}) and CurveFault-B (\textbf{j--l}). Columns show initial normalized gradient variance, final relative parameter deviation from the paired IID model, and paired final excess error against solution-bin transport distance. Points represent five seeds at four concentrations; grey lines connect observations from the same seed. Transport distance and parameter deviation are dimensionless.}\label{fig:publicmechanisms}
\end{figure}

CurveFault-B behaves differently. Its mean paired effect is negative at $\alpha=0.1$ and 0.01; at 0.01 it is $-0.253$ pp (95\% CI, $-0.528$ to 0.022). Meanwhile, initial gradient variance and final parameter separation both increase between $\alpha=10$ and 0.01 (Fig.~\ref{fig:publicmechanisms}). In this inverse task, increased response separation and parameter displacement coexist with an uncertain, slightly negative mean prediction effect. The mean normalized within-bin residual is also larger for CurveFault-B (0.827) than for Darcy (0.109), shallow water (0.144) or cylinder flow (0.192), indicating a coarser ten-bin representation of its response geometry (Supplementary Fig.~5).

Supplementary Figs.~7 and 8 report the checkpoint-level and seed-level paired effects.

\subsection{Bin correction improves Burgers and Darcy prediction}
The bin representation supplies information for optimization as well as partitioning. SB-CFO estimates a gradient prototype for each solution bin at the broadcast model and adds a client-specific correction proportional to the difference between global and local bin proportions. Its sample-weighted mean is zero, and fresh exact prototypes cancel the between-bin component of gradient disagreement (Methods; Supplementary Note~6). The correction is inserted before both Adam moment updates, leaving the model architecture, training samples and sample-count aggregation weights unchanged.

We compared FedAvg and SB-CFO under stratified IID and $\alpha=0.01$ for Burgers with $\nu=0.005$ and Darcy with $\beta=0.01$. The algorithms share the model, partition and training schedule within each seed (Methods). On the non-IID partitions, mean final relative error decreased from 49.52\% to 47.17\% for Burgers and from 64.54\% to 56.03\% for Darcy (Fig.~\ref{fig:sbcfo}). The within-seed improvements were 2.35 pp (95\% CI, 1.20--3.51) and 8.50 pp (4.40--12.60), respectively. All five seeds improved in both tasks.

Each algorithm was also compared with its own IID trajectory. The mean final non-IID penalty decreased from 2.46 to 0.14 pp for Burgers and from 9.92 to 1.34 pp for Darcy. The corresponding paired reductions in penalty were 2.32 pp (95\% CI, 1.18--3.46) and 8.58 pp (4.58--12.57). The paired reduction measures the change in partition sensitivity after accounting for any improvement under IID. The residual SB-CFO penalties have intervals crossing zero, indicating uncertainty in the remaining effect. Under non-IID Darcy, mean SB-CFO error increased from 53.94\% at round 500 to 56.03\% at round 1,000, while FedAvg error decreased over the same interval.

The gradient diagnostics support the intended correction mechanism. At the final model, the geometric mean ratio of gradient variance after correction to variance before correction was 0.220 for Burgers (95\% CI, 0.114--0.426) and 0.055 for Darcy (0.011--0.282). These ratios isolate the immediate effect of the applied correction at the same SB-CFO model. The geometric mean local update variance was also lower than FedAvg at the recorded non-IID checkpoints (Supplementary Fig.~9).

At matched optimizer steps and communication rounds, full-prototype SB-CFO incurs additional resource costs. Relative to paired FedAvg runs, mean communication ratios were 8.15 for Burgers and 7.25 for Darcy, while gradient sample-visit ratios were 2.00 and 3.50. Mean elapsed-time ratios were 2.13 and 2.58 in the sequential simulator.

\begin{figure}[!htbp]
\centering\includegraphics[width=\linewidth,height=.56\textheight,keepaspectratio]{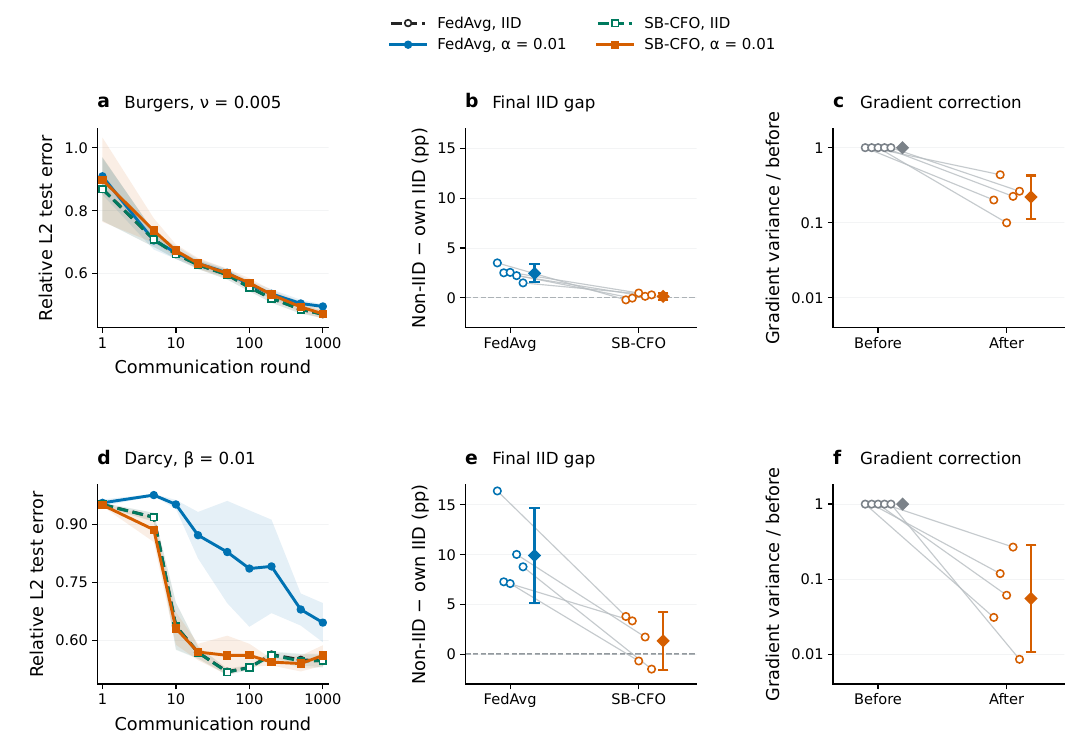}
\caption{SB-CFO links bin-conditioned gradient correction to improved prediction. Rows show Burgers with $\nu=0.005$ (\textbf{a--c}) and Darcy with $\beta=0.01$ (\textbf{d--f}). \textbf{a,d}, Relative $L_2$ test-error curves for FedAvg and SB-CFO under stratified IID and $\alpha=0.01$; shading gives pointwise 95\% Student-$t$ intervals across five seeds. The round axis is logarithmic; lines connect recorded checkpoints. \textbf{b,e}, Final non-IID minus own-IID error in percentage points, calculated within seed. Open circles are seed-level values, grey lines connect matched seeds, and diamonds and bars give arithmetic means and 95\% intervals. \textbf{c,f}, Gradient variance after correction divided by variance before correction at the same SB-CFO global model. Before is normalized to one within each seed; diamonds and bars show geometric means and back-transformed log-space 95\% Student-$t$ intervals. Final comparisons use round 1,000; the applied prototype correction is one round old at that checkpoint.}\label{fig:sbcfo}
\end{figure}

\section{Discussion}
Solution-space partitioning provides a common way to construct heterogeneous clients when PDE datasets have incompatible input descriptions. It connects clustering-based allocation in federated benchmarks with the field-valued targets of operator learning \cite{li2022niid,lin2022fednlp,kovachki2023neuraloperator}. The same response bins define client populations, enter finite-sample and conditional optimization bounds, and supply gradient prototypes for correction. The measured effects combine changes in response composition and client size, with realized severity determined by the task-specific response geometry. Input-based partitions complement this construction by isolating specified physical conditions.

The experiments distinguish gradient disagreement, parameter separation and excess prediction error. Wang et al.\ identify heterogeneous quadratic problems in which large gradient dissimilarity leaves FedAvg convergence unaffected, and characterize the role of the average client update at an optimum \cite{wang2024effectiveness}. Our small or negative prediction gaps are consistent with this distinction, although their causes depend on the learned model and test-risk geometry. The Adam analysis separates the current bin-mixture gradient from retained moments, adaptive preconditioner variation and local path effects. Even with an unchanged pooled objective, client-specific normalization can prevent first-order parameter updates from canceling. The direction of parameter displacement determines its first-order effect on test risk.

Optimization choices consequently belong in the interpretation of any heterogeneity benchmark. Local learning rates can change both convergence speed and the objective effectively followed by local-update methods \cite{charles2021tradeoffs}. Unequal local update counts introduce a further source of objective inconsistency, as the analysis underlying FedNova shows \cite{wang2020fednova}. Fixing local steps controls that source of variation, although unequal client sizes still affect sampling and aggregation. In our ablation, Nesterov SGD has higher absolute errors and larger mean non-IID penalties than Adam at $\alpha=0.01$, but the SGD penalty estimates have wide intervals spanning zero. Interpreting partition sensitivity therefore requires absolute errors and paired effects alongside the optimizer, schedule and training budget \cite{wang2021fieldguide}.

The stronger Burgers effects at lower viscosity also motivate examining which spatial scales the model learns. Spectral-bias analyses show that neural networks can learn different frequencies at different rates \cite{rahaman2019spectral}. Frequency-resolved errors would test whether this contributes to the observed physical-regime dependence.

SB-CFO turns the allocation representation into a corrective signal. FedProx constrains local movement, SCAFFOLD uses client control variates, and FedDyn dynamically regularizes local objectives \cite{li2020fedprox,karimireddy2020scaffold,acar2021feddyn}. SB-CFO factorizes a component of client bias through shared solution-bin gradients and observable bin proportions. The zero weighted mean of its correction preserves the pooled empirical objective, and its prototype error determines how much between-bin heterogeneity remains. Local Adam incorporates this correction into both moments, which differs from adaptive server optimization \cite{reddi2021adaptive}. The Burgers and Darcy experiments show that this bin-conditioned correction improves prediction relative to paired FedAvg.

The dispersion bounds show how within-bin gradient differences, accumulated moments and client-specific preconditioners limit the effect of bin correction. Adam counterexamples and sufficient-condition analyses motivate this explicit treatment \cite{reddi2018adam,chen2019adamtype,defossez2022simple}. Quantitative application of the stationarity bound requires estimates of the state-dependent terms and a step size satisfying its stated condition.

Full-gradient prototypes currently impose substantial communication overhead. The comparisons match optimizer steps and communication rounds, while SB-CFO additionally computes and transmits conditional gradients. Less frequent refreshes, subsampled bin estimates and compressed prototype transmission would reduce these costs but increase estimation or staleness error. A direct full-client-gradient correction, $\vect{r}_k=\vect{g}-\vect{g}_k$, would use the same full-data gradient access without bin prototypes. Matched-budget comparisons with this control, perturbed bins and other corrective algorithms would clarify the value of response-conditioned information.

The protocol constructs synthetic client populations from labelled simulation data before federated training. Deployment introduces additional statistical and systems constraints \cite{kairouz2021advances,wang2021fieldguide}. Matching response-bin proportions controls composition and size differences while retaining within-bin variation in the joint input--response distribution. Natural institutional partitions would complement this stratified empirical comparison.

Response discretization determines which physical differences the protocol resolves. Ten Euclidean bins represent some tasks more faithfully than others, as the larger residual for CurveFault-B illustrates. Mesh weighting, channel scaling and bin count can change the geometry. MeshGraphNets and recent free-boundary neural operators demonstrate learned simulation on adaptive meshes and evolving domains \cite{pfaff2021mesh,long2026freeboundary}. Extending our construction to such settings would require a common field representation or a distance that accounts for the domain geometry. Burgers reference solutions remain sensitive to spatial resolution at low viscosity, so the observed learning trends describe the evaluated finite-grid response family.

Broader physics collections offer a route to testing whether these relationships persist across additional spatial scales and dynamics \cite{ohana2024well}. Comparing response metrics, discretization fidelity and model architecture on the same data would help determine when response geometry predicts a learning penalty and when it mainly reveals different optimization trajectories with similar predictive accuracy.

\FloatBarrier
\section{Methods}
Vectors and matrices are denoted by bold lower- and upper-case symbols, respectively; vector norms are Euclidean unless otherwise stated. Supplementary Note~1 lists the notation.

\subsection{Response representation and discretization}
For a supervised dataset $\cD=\{(\vect{x}_i,\vect{y}_i)\}_{i=1}^N$, $\vect{x}_i$ contains the available input fields and physical conditions, and $\vect{y}_i$ is the target field. In forward problems $\vect{y}_i$ is a numerical PDE response; in CurveFault-B it is the subsurface velocity field to be recovered. Preprocessing, bins and distances are fitted exclusively to training responses.

For the public datasets, inputs and targets are first standardized channel-wise using training-set means and standard deviations over examples and spatial locations. The standard deviation is floored at $10^{-6}$. Distinct prediction horizons occupy separate target channels. Controlled targets retain their original scale. Let $\widetilde{\vect{y}}_i$ be the resulting flattened target vector. We construct
\begin{equation}
\vect{z}_i=\frac{\widetilde{\vect{y}}_i-\overline{\vect{y}}}{s_y},\quad
\overline{\vect{y}}=\frac1N\sum_i\widetilde{\vect{y}}_i,\quad
s_y^2=\frac1N\sum_i\|\widetilde{\vect{y}}_i-\overline{\vect{y}}\|_2^2.
\label{eq:normalization}
\end{equation}
For nonconstant responses, this global scale removes overall response magnitude and uniform-grid vector-length scaling. The implementation rejects a training pool if the computed scale is at most $10^{-12}$. Channel standardization in the public tasks additionally determines the relative weight of physical components and horizons.

We use $k$-means++ initialization \cite{arthur2007kmeans} followed by Lloyd iterations to obtain $B=10$ bins with labels $b(i)$ and centres $\vect{c}_b$. Iteration stops when assignments are unchanged or the limit is reached: 50 iterations for controlled tasks and 100 for public tasks. Empty clusters are reinitialized using distant examples. The mean residual $\varepsilon=N^{-1}\sum_i\|\vect{z}_i-\vect{c}_{b(i)}\|_2$ records discretization fidelity. Supplementary Proposition~1 bounds the discrepancy between empirical response transport and centroid-based transport using the client-specific residuals.

\subsection{Non-IID allocation and the stratified reference}
Independently for each bin, draw $\mat{R}_{\cdot b}\sim\mathrm{Dirichlet}(\alpha\ones_K)$ and then $(m_{1b},\ldots,m_{Kb})\sim\mathrm{Multinomial}(N_b,\mat{R}_{\cdot b})$. Randomly permuting bin members and dividing them into blocks of these sizes assigns each example exactly once. We evaluate $\alpha\in\{10,1,0.1,0.01\}$. If a client has fewer than 16 examples, the repair procedure transfers one randomly selected example from a donor with surplus above 16 to a smallest client. Donors are sampled in proportion to their surplus, and transfers repeat until all sizes satisfy the constraint. The procedure induces both response-composition and quantity heterogeneity.

For the IID reference, every client first receives $\lfloor N_b/K\rfloor$ members from each bin. Remaining members are assigned to distinct clients with the largest outstanding total capacity, with seeded random tie-breaking. Client capacities are $\lfloor N/K\rfloor$ or $\lceil N/K\rceil$. This yields per-bin counts differing by at most one between clients and total sizes differing by at most one. The target mixture is the global empirical $\vect{\pi}$. Finite counts introduce rounding differences, particularly for bins with fewer members than clients. This stratified reference uses disjoint samples and preserves global bin proportions to integer precision.

\subsection{Geometric and optimization diagnostics}
For client sizes $N_k=\sum_bm_{kb}$ and weights $p_k=N_k/N$, define the histogram entries $q_{kb}=m_{kb}/N_k$. The cost matrix $\mat{C}$ has entries $C_{bc}=\|\vect{c}_b-\vect{c}_c\|_2$. We compute discrete transport \cite{peyre2019ot} by linear programming:
\begin{equation}
W_{\mat{C}}(\vect{q}_i,\vect{q}_j)=\min_{\mat{T}\geq0}\sum_{b,c}T_{bc}C_{bc},
\quad \mat{T}\ones_B=\vect{q}_i,\quad \mat{T}^\top\ones_B=\vect{q}_j,
\quad
D_{\mathrm{sol}}=\frac{2}{K(K-1)}\sum_{i<j}W_{\mat{C}}(\vect{q}_i,\vect{q}_j).
\label{eq:distance}
\end{equation}
Here $\mat{T}\in\R^{B\times B}$ is a nonnegative transport matrix. The diagnostic is an unweighted mean over client pairs. The weighted squared quantities used in the theory are defined separately in Supplementary Note~4.

Let $\mathcal I_k$ contain the sample indices on client $k$, and $\mathcal I_b=\{i:b(i)=b\}$ the indices in bin $b$. For the common per-example squared-error loss $\ell_i(\vect{\theta})=\ell(\vect{\theta};\vect{x}_i,\vect{y}_i)$, define
\begin{equation}\label{eq:objectives}
\begin{aligned}
F_k(\vect{\theta})&=\frac1{N_k}\sum_{i\in\mathcal I_k}\ell_i(\vect{\theta}),\\
F_b^{\mathrm{bin}}(\vect{\theta})&=\frac1{N_b}\sum_{i\in\mathcal I_b}\ell_i(\vect{\theta}),\\
F(\vect{\theta})&=\frac1N\sum_{i=1}^N\ell_i(\vect{\theta})
=\sum_kp_kF_k(\vect{\theta})
=\sum_b\pi_bF_b^{\mathrm{bin}}(\vect{\theta}).
\end{aligned}
\end{equation}
Thus $\vect{g}_k=\nabla F_k$, $\vect{\mu}_b=\nabla F_b^{\mathrm{bin}}$ and $\vect{g}=\nabla F$, evaluated at the stated model. At common initialization, we evaluate full empirical client gradients $\vect{g}_k$ and their sample-weighted mean $\vect{g}=\sum_kp_k\vect{g}_k$. The plotted gradient diagnostic is
\begin{equation}
V_g=\frac{\sum_kp_k\|\vect{g}_k-\vect{g}\|_2^2}
{\max\{\sum_kp_k\|\vect{g}_k\|_2^2,10^{-12}\}}.
\label{eq:normalized_gradient}
\end{equation}
Final parameter deviation is $D_\theta=\|\vect{\theta}_\alpha-\vect{\theta}_{\mathrm I}\|_2/\max\{\|\vect{\theta}_{\mathrm I}\|_2,10^{-12}\}$. For complex FNO parameters, squared norms use squared magnitudes.

Test error is the global relative norm $e=\|\widehat{\mat{Y}}-\mat{Y}\|_F/\|\mat{Y}\|_F$, where rows index evaluated examples and columns stack all target channels and grid points. Public predictions are returned to the original target units before evaluation. We compute $\Delta_{\mathrm{err}}=100(e_\alpha-e_{\mathrm I})$ within each seed.

\subsection{Training and task configuration}
The partition benchmark uses sample-weighted FedAvg \cite{mcmahan2017fedavg} with ten fully participating clients, five local steps per round and 1,000 rounds. Each step draws a minibatch without replacement within that step; samples can recur across steps. Batch size is 16 for all controlled tasks and Darcy, and 8 for the remaining public tasks, capped by local dataset size. The principal optimizer is Adam \cite{kingma2015adam} with cosine learning-rate decay from $10^{-3}$ to $10^{-4}$. Optimizer states persist locally across rounds while model parameters are replaced by the broadcast aggregate. The optimizer ablation uses Nesterov SGD with momentum 0.9 and cosine decay from $10^{-2}$ to $10^{-3}$, without weight decay. The client-count ablation uses twenty clients; the remaining training settings follow the principal configuration. Each IID reference matches the model, optimizer, schedule, local steps, rounds and sample pool of its non-IID counterpart. Its trajectory is reused across concentrations within a seed.

Controlled operators use Chebyshev input expansions with ten independent coefficients sampled uniformly from $[-1,1]$. The antiderivative integrates the input on $[0,1]$. Diffusion--reaction uses $u_t=0.01u_{xx}+0.01u^2+f(x)$ with zero initial and boundary values. Burgers uses $u_t+uu_x=\nu u_{xx}$ with periodic extension of the sampled initial field. DeepONet branch and trunk networks use rectified linear unit (ReLU) activations, with width/depth settings 40/2, 100/3 and 64/2, respectively. Numerical labels are generated by the finite-grid solvers described in Supplementary Note~1.

The public tasks use PDEBench Darcy at $\beta=0.01$, PDEBench shallow-water height prediction at time indices 25, 50, 75 and 100 from index 0, CFDBench cylinder prediction with four history frames and three future horizons, and OpenFWI CurveFault-B inversion. In Darcy, $\beta$ specifies the constant forcing amplitude \cite{takamoto2022pdebench}. Cylinder cases are split before windows are constructed. The first three tasks use a four-layer FNO with width 32 and 12 Fourier modes; CurveFault-B uses a compact convolutional InversionNet-inspired model \cite{wu2020inversionnet}. Supplementary Note~1 gives the data files, splits, preprocessing and architectures.

\subsection{Solution-Bin Controlled Federated Optimization}
The algorithm uses the same fixed training-response labels as the partition protocol. Each round has four steps: clients compute bin gradients at the broadcast model; the server pools and broadcasts the prototypes; clients apply their bin-mixture correction before local Adam updates; and the server averages the updated models while clients retain their moment states. Set $\mathcal I_{kb}=\{i\in\mathcal I_k:b(i)=b\}$ and $m_{kb}=|\mathcal I_{kb}|$. At the start of round $t$, each occupied client--bin pair computes
\begin{equation}
\widehat{\vect{\mu}}_{kb}^t=\frac1{m_{kb}}\sum_{i\in\mathcal I_{kb}}\nabla_{\vect{\theta}}\ell(\vect{\theta}_t;\vect{x}_i,\vect{y}_i),\quad
\widehat{\vect{\mu}}_b^t=\sum_{k:m_{kb}>0}\frac{m_{kb}}{N_b}\widehat{\vect{\mu}}_{kb}^t.
\label{eq:method_prototype}
\end{equation}
Here $\ell$ is mean squared error over the output coordinates of one example. The server combines conditional gradient vectors using their bin counts; raw examples remain local during training. Broadcasting the prototypes allows each client to construct
\begin{equation}
\vect{r}_k^t=\sum_b(\pi_b-q_{kb})\widehat{\vect{\mu}}_b^t,\quad
\sum_kp_k\vect{r}_k^t=\zeros.
\label{eq:method_correction}
\end{equation}
The evaluated configuration uses all training members of each occupied bin, refreshes prototypes every round, and sets correction strength to one. Client histograms, counts and bin assignments remain fixed throughout training.

Each local minibatch gradient is replaced by $\widetilde{\vect{g}}_{k,e}=\nabla\ell_{\mathcal B_{k,e}}(\vect{\theta}_{k,e})+\vect{r}_k^t$. Adam updates its first moment with $\widetilde{\vect{g}}_{k,e}$ and its second moment with $\widetilde{\vect{g}}_{k,e}^{\odot2}$. With persistent step count $j=tE+e+1$, the exact updates are
\begin{equation}
\begin{aligned}
\vect{m}_{k,e+1}&=\beta_1\vect{m}_{k,e}+(1-\beta_1)\widetilde{\vect{g}}_{k,e},\\
\vect{v}_{k,e+1}&=\beta_2\vect{v}_{k,e}+(1-\beta_2)\widetilde{\vect{g}}_{k,e}^{\odot2},\\
\vect{\theta}_{k,e+1}&=\vect{\theta}_{k,e}-\eta_t
\frac{\vect{m}_{k,e+1}/(1-\beta_1^j)}{\sqrt{\vect{v}_{k,e+1}/(1-\beta_2^j)}+\epsilon_{\mathrm A}\ones_{d_\theta}}.
\end{aligned}
\label{eq:method_adam}
\end{equation}
All vector operations in the fraction are coordinatewise. We use $\beta_1=0.9$, $\beta_2=0.999$, $\epsilon_{\mathrm A}=10^{-8}$, no weight decay and no gradient clipping. Complex FNO weights are handled through their real and imaginary coordinates. The server sets $\vect{\theta}_{t+1}=\sum_kp_k\vect{\theta}_{k,E}$; clients retain their own moment states when the next global model is broadcast. These client-local Adam updates differ from the server-side adaptation used in FedAdam \cite{reddi2021adaptive}.

For Burgers at $\nu=0.005$ and Darcy at $\beta=0.01$, all four algorithm--partition conditions share the model, data, initialization, local-step budget and learning-rate schedule within seed. FedAvg and SB-CFO additionally share each client minibatch stream within a fixed partition. For $t=0,\ldots,T-1$,
\begin{equation}
\eta_t=\eta_{\min}+\frac{\eta_{\max}-\eta_{\min}}2
\left[1+\cos\left(\frac{\pi t}{T-1}\right)\right],\quad T=1000,\quad E=5.
\label{eq:cosine}
\end{equation}
This schedule is also used in the partition benchmark: $(\eta_{\max},\eta_{\min})=(10^{-3},10^{-4})$ for Adam and $(10^{-2},10^{-3})$ for the Nesterov SGD ablation. The rate is held fixed within a round. The dedicated Adam comparison uses training/test counts of 800/500 for Burgers and 2,000/500 for Darcy, with batch size 16 in both tasks. It evaluates IID and $\alpha=0.01$ on ten fully participating clients. Each algorithm has its own matched IID trajectory.

For gradient correction, we evaluate full client gradients at a common current model before and after adding the applied $\vect{r}_k^t$. The diagnostic is the ratio of their unnormalized sample-weighted variances. At the final checkpoint, no further training refresh is performed, so the applied correction comes from round 999 and has age one. Update variance is $\sum_kp_k\|\Delta\vect{\theta}_k-\sum_jp_j\Delta\vect{\theta}_j\|^2$ for the preceding round. Gradient sample visits count optimizer and prototype examples, excluding diagnostic-only visits. Communication counts two model transfers per client and round, one gradient vector per occupied client--bin upstream, and all $B$ prototype vectors sent to each client downstream. These are logical uncompressed tensor payloads; count metadata, partition construction and protocol overhead are excluded. Recorded wall-clock time includes evaluation and diagnostics in the sequential simulator. The evaluated protocol does not apply secure aggregation or differential privacy to transmitted counts and gradients.

\subsection{Replication and statistical analysis}
The five seed values are 0, 1, 42, 999 and 2026. Controlled-task seeds vary generated training and test data, clustering, initialization and allocation; all concentrations within a seed share the same data and initialization. Public tasks use one fixed train--test split and clustering seed 42, while the five evaluation seeds vary initialization, allocation and training randomness. Their uncertainty is therefore conditional on the selected dataset split and codebook.

We record rounds 0, 1, 5, 10, 20, 50, 100, 200, 500 and 1,000. Final comparisons use round 1,000. For five replicate values $d_s$, the reported interval is $\overline d\pm t_{0.975,4}s_d/\sqrt5$, with sample standard deviation $s_d$ and $t_{0.975,4}\simeq2.776$. For excess error, $d_s$ is the within-seed difference. SB-CFO comparisons compute prediction improvement and reduction of the own-IID gap within each seed before averaging. Positive gradient-variance ratios and update-variance curves use geometric means with Student-$t$ intervals on logarithms, transformed back to the original scale. Cost summaries use arithmetic means of within-seed SB-CFO/FedAvg ratios. The Student-$t$ intervals assume approximately normal replicate values and are pointwise, without adjustment for multiple comparisons. Concentrations sharing a seed and IID reference are treated as repeated conditions. Supplementary Note~10 describes numerical checks of the theoretical identities and PDE reference solutions.

\subsection{AI-assisted manuscript preparation}
OpenAI Codex assisted with language editing, organization of the manuscript, consistency checks against the experiment files, verification of mathematical derivations and numerical-reference checks.

\section*{Data availability}
The antiderivative, Burgers and diffusion--reaction datasets were generated in this study from sampled Chebyshev input functions using numerical integration or PDE solvers.
Generation procedures, discretization settings and random seeds are given in Supplementary Note~1.
These datasets can be regenerated using the numerical solvers in \texttt{operator\_tasks/tasks.py} and the configurations in \texttt{config/} of the archived code at \url{https://doi.org/10.5281/zenodo.23686922}.
PDEBench data are available from \url{https://doi.org/10.18419/darus-2986}.
CFDBench data and documentation are available from \url{https://github.com/luo-yining/CFDBench}.
OpenFWI data and split annotations are available from \url{https://openfwi-lanl.github.io/}.
The selected subsets and preprocessing are specified in Methods and Supplementary Note~1.
Source Data accompanying this paper contain the seed-level observations, partition diagnostics and summaries underlying the figures and tables. Recorded experiment outputs are also available in the \texttt{results/} directory of the code archive.

\section*{Code availability}
The code for data generation, federated training and analysis is publicly available under the MIT license at \url{https://github.com/NUDTPingLuo/Fed-PDE}.
Version 1.1.0, including the experiment configurations and recorded results used in this study, is archived on Zenodo at \url{https://doi.org/10.5281/zenodo.23686922}. Additional scripts for verifying public-data provenance and checking Burgers numerical references accompany the Source Data.

\section*{Funding}
This work was supported by the National Natural Science Foundation of China under Grant Nos.\ 62522610, 62376278 and 62421002, and NUDT Foundational Research Funding (JS25-02).

\section*{Author contributions}
K.X., D.L. and T.S. conceived and supervised the research.
P.L. developed the methodology and performed the theoretical analysis.
P.L., J.W. and Z.W. implemented the code, conducted the experiments and analysed the results.
P.L. wrote the manuscript.

\section*{Competing interests}
The authors declare no competing interests.

\section*{Additional information}
Supplementary Information is included after the references in this document. Correspondence should be addressed to Kele Xu, Dongsheng Li or Tao Sun.
\setlength{\bibsep}{5pt plus 1pt minus 1pt}
\bibliography{references}
\clearpage
\section*{Supplementary Information}
\setcounter{equation}{0}
\setcounter{figure}{0}
\setcounter{table}{0}
\renewcommand{\theequation}{S\arabic{equation}}
\renewcommand{\figurename}{Supplementary Figure}
\renewcommand{\tablename}{Supplementary Table}
\renewcommand{\theHequation}{supp.\arabic{equation}}
\renewcommand{\theHfigure}{supp.\arabic{figure}}
\renewcommand{\theHtable}{supp.\arabic{table}}
\section*{Supplementary Note 1. Data generation and model implementation}
\addcontentsline{toc}{section}{Supplementary Note 1. Data generation and model implementation}
\paragraph{Shared notation.}
Bold lower- and upper-case symbols denote vectors and matrices. Gradients are column vectors in $\R^{d_\theta}$; response vectors lie in $\R^{d_y}$. Vector norms are Euclidean; $\|\cdot\|_{\mathrm{op}}$ and $\|\cdot\|_F$ denote the induced Euclidean operator and Frobenius norms. The symbols $\ones_d,\zeros_d$ and $\mat{I}_d$ denote the all-ones vector, zero vector and identity matrix, with dimensions omitted when clear from context.

\begin{center}
\small
\begin{tabular}{p{.29\linewidth}p{.65\linewidth}}
\toprule
Symbol & Meaning and dimension\\
\midrule
$K,B,N,N_k,N_b$ & Numbers of clients, bins, pooled samples, client samples and bin members.\\
$\mat{R},\mat{M}$ & $K\times B$ allocation-probability and count matrices, with scalar entries $R_{kb},m_{kb}$.\\
$\vect{u}_k,\vect{\omega}_k$ & Raw population and count profiles in $\R^B$: $u_{kb}=K\pi_bR_{kb}$ and $\omega_{kb}=Km_{kb}/N$.\\
$s_k,s_{\min},s_{\max},\rho$ & Client mass relative to its mean, its extrema and its root-mean-square deviation from one.\\
$\vect{q}_k,\vect{\pi},\vect{p}$ & Client and pooled bin proportions in $\R^B$; sample weights $\vect{p}\in\R^K$.\\
$\vect{z}_i,\vect{c}_b,\mat{C}$ & Normalized response, centroid and $B\times B$ transport-cost matrix.\\
$F_k,F_b^{\mathrm{bin}},F$ & Client, bin-conditioned and pooled scalar losses.\\
$\vect{g}_k,\vect{\mu}_b,\vect{g}$ & Client, bin-conditioned and pooled gradients in $\R^{d_\theta}$.\\
$\mat{G},\widehat{\mat{G}}$ & Exact and estimated bin-gradient matrices in $\R^{d_\theta\times B}$.\\
$\vect{r}_k,\vect{a}_k$ & Bin correction and corrected full client gradient.\\
$\widetilde{\vect{g}}_{k,e},\vect{\xi}_{k,e}$ & Corrected minibatch gradient and raw sampling noise.\\
$\vect{\theta},\vect{m},\vect{v},\mat{D}$ & Model parameters, Adam moments and diagonal preconditioner.\\
$\cH_q,\cH_g,\cH_a$ & Sample-weighted histogram, raw-gradient and corrected-gradient variances.\\
$D_{\mathrm{sol}},D_{\mathrm{loc}},D_\theta$ & Transport diagnostic, local endpoint dispersion and relative parameter deviation.\\
$e,\Delta_{\mathrm{err}},\vect{e}_t$ & Relative test error, paired excess error in pp and the Adam round-error vector.\\
$\vect{\chi}_t$ & Joint state containing global parameters and all client moment states.\\
\bottomrule
\end{tabular}
\end{center}

\subsection*{Controlled operator datasets}
Inputs are sampled as $f(x)=\sum_{j=0}^{9}a_jT_j(2x-1)$ on $x\in[0,1]$, where $a_j$ are independent uniform variables on $[-1,1]$ and $T_j$ are Chebyshev polynomials. The antiderivative uses 100 input and output locations. Its target $u(x)=\int_0^xf(\xi)\,d\xi$ is evaluated by cumulative trapezoidal integration. Each seed generates 1,000 training examples and 1,000 test examples.

Diffusion--reaction uses 101 spatial and 101 temporal locations on $[0,1]^2$. Its equation is
\begin{equation}
u_t=0.01u_{xx}+0.01u^2+f(x),\quad u(x,0)=0,\quad u(0,t)=u(1,t)=0.
\end{equation}
The implementation uses centred spatial differences and forward Euler time integration. Each output interval of length $0.01$ is divided into an integer number of substeps so that $\Delta t\leq0.4\Delta x^2/0.01$. The stored solution contains 10,201 values per example. Training and test sets each contain 1,000 examples.

Burgers uses $u_t+(u^2/2)_x=\nu u_{xx}$ on the same space--time grid. The 101 sampled input locations include both endpoints. The solver advances the first 100 values with periodic indexing and stores a duplicate of the first value at the last spatial location. Thus, the numerical initial condition is the periodic extension of the sampled field; the Chebyshev expansion itself need not satisfy periodic endpoint compatibility. Convective fluxes use the local Rusanov formula \cite{leveque2002finitevolume}
\begin{equation}
\widehat f_{j+1/2}=\frac{f(u_j)+f(u_{j+1})}{2}
-\frac{\max(|u_j|,|u_{j+1}|)}{2}(u_{j+1}-u_j),\quad f(u)=u^2/2.
\end{equation}
Diffusion uses the centred second difference. At the start of each output interval, the substep count is chosen from
\begin{equation}
\Delta t_{\max}=\frac{0.35}{\max|u|/\Delta x+2\nu/\Delta x^2}.
\end{equation}
The maximum is taken over the states in the solution batch. The principal viscosity is $0.1$; the physical-regime study also uses $0.05$, $0.01$ and $0.005$. Each seed has 800 training and 500 test examples. All labels are finite-grid numerical references.

For evaluation seed $s$, the controlled training, test, clustering and partition seeds are respectively $42+10,000+s$, $42+50,000+s$, $42+20,000+s$ and $42+30,000+s$. Concentrations share their seed-specific data, bins and model initialization. The branch network of DeepONet \cite{lu2021deeponet} receives the sampled input function; the trunk receives a spatial coordinate for the antiderivative and a space--time coordinate for the two PDEs. Their outputs are combined by a dot product with an output bias. ReLU networks use width/depth settings 40/2, 100/3 and 64/2 for antiderivative, diffusion--reaction and Burgers. Training minimizes mean squared error over the predicted solution values.

\subsection*{Public simulation datasets}
Darcy uses \texttt{2D\_DarcyFlow\_beta0.01\_Train.hdf5} from PDEBench \cite{takamoto2022pdebench}. The input is permeability and the target is pressure. A seeded permutation (seed 2026) selects disjoint sets of 2,000 training and 500 test examples. A spatial stride of two is applied to both input and target arrays. PDEBench defines $\beta$ as the constant source in the elliptic problem, with zero boundary values:
\begin{equation}
-\nabla\!\cdot(a(\vect{x})\nabla u_\beta(\vect{x}))=\beta,\quad u_\beta|_{\partial\Omega}=0.
\end{equation}
For a fixed permeability field and a unique solution, linearity gives $u_\beta=\beta u_1$. With matched fields and exact solutions, this positive amplitude rescaling cancels under target standardization. 

Shallow water uses the PDEBench radial dam-break file \texttt{2D\_rdb\_NA\_NA.h5}. A permutation with seed 2026 splits the 1,000 available trajectories into 800 training and 200 test examples. The input is the height field at stored time index 0; targets are height fields at indices 25, 50, 75 and 100. All four horizons are concatenated as output channels and predicted jointly. Spatial stride is one. Time indices refer to stored snapshots.

Cylinder flow uses the available \texttt{bc}, \texttt{geo} and \texttt{prop} subsets of CFDBench \cite{luo2023cfdbench}. The 186 physical cases are split into 166 training and 20 test cases with seed 2026 before temporal windows are constructed. Each example contains four consecutive frames of both velocity components, plus five condition channels: a geometry-mask channel, inlet velocity, density, base-10 logarithm of viscosity, and cylinder radius. Targets contain both velocity components at offsets 1, 20 and 100 relative to the last input frame. Window anchors are separated by 64 frames, giving 2,784 training and 300 test windows. Spatial stride is one. The recorded Reynolds-number range is 4--1,000. Allocation to clients is at the window level, so training windows from one physical case may occur on different clients. The geometry-mask channel and velocity fields share the same spatial convention: rows correspond to $y$ and columns to $x$. The circular mask is evaluated on the physical coordinate grid defined by each case.

CurveFault-B uses OpenFWI \cite{deng2022openfwi} seismic shot gathers as inputs and subsurface velocity maps as targets. The loader reads 2,000 training examples and 500 examples from the official validation annotation, in annotation order; the latter form the evaluation set reported as test data here. The annotation files are \texttt{curvefault\_b\_train.txt} and \texttt{curvefault\_b\_val.txt}. Input time samples are subsampled with stride four, receiver stride is one, and target velocity maps are retained.

Public splits and response codebooks are fixed across evaluation seeds. Clustering uses seed 42; the allocation seed is $30,000+s$. Initializations and minibatch sampling vary with $s$. Channel-wise normalization parameters are fitted on training examples only, then reused for the test set. Public losses use standardized targets; reported errors use denormalized predictions and targets.

\subsection*{Public-task architectures}
Darcy, shallow water and cylinder flow use a two-dimensional FNO \cite{li2021fno}. Two normalized coordinate channels are appended to the input before a $1\times1$ convolution lifts it to width 32. Each of four blocks sums a spectral convolution and a pointwise linear convolution, followed by a Gaussian error linear unit (GELU). Spectral convolution retains 12 modes per direction using separate complex weights for positive and negative first-axis frequencies. A $1\times1$ projection expands the width to 64, applies GELU, and maps to the required output channels. Each model predicts its configured output channels in one forward pass.

The CurveFault-B model (InversionNet-lite) uses a compact encoder--decoder inspired by InversionNet \cite{wu2020inversionnet}. Four stride-two $3\times3$ convolutional layers have widths 24, 48, 96 and 96, each followed by eight-group normalization and GELU. Adaptive average pooling produces an $18\times18$ latent grid. The decoder uses $3\times3$ convolutions with widths 48 and 24, GELU activations and a final $1\times1$ target projection. Bilinear interpolation maps this output to the target grid.

\subsection*{Training configuration}
All clients participate in each communication round. Model parameters are sample-weighted at aggregation, as in FedAvg \cite{mcmahan2017fedavg}, and each client retains its own optimizer state. Minibatches are sampled independently at each local step, without replacement within the batch. The shared training budget is five local steps and 1,000 rounds. The partition benchmark and the FedAvg--SB-CFO comparisons use Adam \cite{kingma2015adam} with cosine decay from $10^{-3}$ to $10^{-4}$, $\beta_1=0.9$, $\beta_2=0.999$, $\epsilon_{\mathrm A}=10^{-8}$ and no weight decay. The cosine schedule follows the annealing form of SGDR \cite{loshchilov2017sgdr}, without restarts. Batch size is 16 for controlled tasks and Darcy, and 8 for the other public tasks. The dedicated Burgers and Darcy comparisons also use batch size 16, with training/test counts 800/500 and 2,000/500. The controlled optimizer ablation uses Nesterov SGD with momentum 0.9, zero dampening, no weight decay and cosine decay from $10^{-2}$ to $10^{-3}$. All rates are constant within a round. Increasing the client count from ten to twenty changes the total number of local updates under this protocol.

Training was conducted on an NVIDIA GeForce RTX 5090 GPU with Python 3.10.22, PyTorch 2.7.1+cu128 and NumPy 1.26.4. Clients were simulated sequentially. Elapsed-time ratios compare paired runs within each task and include evaluation and diagnostic work.

Experiment configurations are included in the code archive. Source Data provide reconstructed public-sample identities, case and window indices, train--test membership, response-bin labels, client assignments and codebooks.

\section*{Supplementary Note 2. Population and finite-sample allocation moments}
\addcontentsline{toc}{section}{Supplementary Note 2. Population and finite-sample allocation moments}
Dirichlet allocations provide a tunable model of client composition in federated benchmarks \cite{hsu2019noniid}; here the categories are solution bins. Fix $K\geq2$ clients and $B\geq2$ nonempty bins. Write $N=\sum_bN_b$, $\pi_b=N_b/N$ and $\vect{\pi}=(\pi_1,\ldots,\pi_B)^\top$. The matrices $\mat{R},\mat{M}\in\R^{K\times B}$ collect allocation probabilities and integer counts, respectively: $(\mat{M})_{kb}=m_{kb}$. Allocation vectors $\mat{R}_{\cdot b}$ are independent across bins and have distribution $\mathrm{Dirichlet}(\alpha\ones_K)$, with $\alpha>0$. Expectations in this note are conditional on the fixed bin populations.

\begin{theorem}[Allocation moments before minimum-size repair]\label{thm:allocation}
For $u_{kb}=K\pi_bR_{kb}$ and $\cH_u=K^{-1}\sum_k\|\vect{u}_k-\vect{\pi}\|^2$,
\begin{equation}\label{eq:rawlaw}
\E\cH_u=\frac{K-1}{K\alpha+1}\sum_b\pi_b^2.
\end{equation}
If $\mat{M}_{\cdot b}\mid \mat{R}_{\cdot b}\sim\mathrm{Multinomial}(N_b,\mat{R}_{\cdot b})$, define $\omega_{kb}=Km_{kb}/N$ and $\cH_m=K^{-1}\sum_k\|\vect{\omega}_k-\vect{\pi}\|^2$. Then
\begin{equation}\label{eq:countlaw}
\E\cH_m=\frac{K-1}{K\alpha+1}\left(\sum_b\pi_b^2+\frac{K\alpha}{N}\right).
\end{equation}
The first expectation is strictly decreasing in $\alpha$. The second is nonincreasing, and is strictly decreasing when at least one $N_b>1$.
\end{theorem}
\begin{proof}
The marginal Dirichlet moments are \cite{mosimann1962compound}
\begin{equation}
\E R_{kb}=\frac1K,\quad
\Var(R_{kb})=\frac{K-1}{K^2(K\alpha+1)},\quad
\E[R_{kb}(1-R_{kb})]=\frac{(K-1)\alpha}{K(K\alpha+1)}.
\end{equation}
Thus $\E u_{kb}=\pi_b$ and $\Var(u_{kb})=(K-1)\pi_b^2/(K\alpha+1)$. Summing these variances over bins and averaging over clients proves equation~\eqref{eq:rawlaw}; the squared Euclidean norm contains only coordinate squares.

The law of total variance gives
\begin{align}
\Var(m_{kb})
&=\E\Var(m_{kb}\mid R_{kb})+\Var(\E[m_{kb}\mid R_{kb}])\\
&=N_b\E[R_{kb}(1-R_{kb})]+N_b^2\Var(R_{kb})\\
&=\frac{N_b(K-1)(N_b+K\alpha)}{K^2(K\alpha+1)}.
\end{align}
Since $\E \omega_{kb}=\pi_b$, multiplication by $K^2/N^2$ and summation yield equation~\eqref{eq:countlaw}. Its derivative is
\begin{equation}
\frac{d}{d\alpha}\E\cH_m
=\frac{K(K-1)}{(K\alpha+1)^2}\left(\frac1N-\sum_b\pi_b^2\right)\leq0.
\end{equation}
The inequality follows from $\sum_bN_b^2\geq\sum_bN_b=N$, with strict inequality when a bin has more than one member. Differentiating equation~\eqref{eq:rawlaw} proves its strict decrease. Finally, the finite-sample expression tends to $(K-1)/N$ as $\alpha\to\infty$.
\end{proof}

Raw profiles have the correct arithmetic mean but need not sum to one within a client. The following result distinguishes this raw-mass variability from sample-weighted histogram variability. It applies to population masses and to realized counts before or after repair, provided every client has positive mass.

\begin{lemma}[Client-mass normalization]\label{lem:mass}
Let $\vect{\pi}\geq\zeros$ with $\ones_B^\top\vect{\pi}=1$. Suppose $\vect{u}_k\geq\zeros$, $K^{-1}\sum_k\vect{u}_k=\vect{\pi}$ and $s_k=\ones_B^\top\vect{u}_k$. Define $s_{\min}=\min_k s_k>0$, $s_{\max}=\max_k s_k$ and $\rho^2=K^{-1}\sum_k(s_k-1)^2$. Set $\vect{q}_k=\vect{u}_k/s_k$, $p_k=s_k/K$ and $\cH_q=\sum_kp_k\|\vect{q}_k-\vect{\pi}\|^2$. Then $\sum_kp_k\vect{q}_k=\vect{\pi}$ and
\begin{equation}\label{eq:massbounds}
\frac{(\sqrt{\cH_u}-\rho\|\vect{\pi}\|)_+^2}{s_{\max}}
\leq\cH_q\leq
\frac{(\sqrt{\cH_u}+\rho\|\vect{\pi}\|)^2}{s_{\min}},
\end{equation}
where $(a)_+=\max(a,0)$.
\end{lemma}
\begin{proof}
The identity follows from $p_k\vect{q}_k=\vect{u}_k/K$. Decompose
$\vect{u}_k-\vect{\pi}=s_k(\vect{q}_k-\vect{\pi})+(s_k-1)\vect{\pi}$.
In the product norm $\|\vect{a}\|_K^2=K^{-1}\sum_k\|\vect{a}_k\|^2$, the mass-deviation term has norm $\rho\|\vect{\pi}\|$. The triangle and reverse triangle inequalities give
\begin{equation}
(\sqrt{\cH_u}-\rho\|\vect{\pi}\|)_+
\leq\|(s_k(\vect{q}_k-\vect{\pi}))_{k=1}^K\|_K
\leq\sqrt{\cH_u}+\rho\|\vect{\pi}\|.
\end{equation}
Because $s_k^2=s_k\cdot s_k$ and $s_{\min}\leq s_k\leq s_{\max}$,
$s_{\min}\cH_q\leq\|(s_k(\vect{q}_k-\vect{\pi}))_{k=1}^K\|_K^2\leq s_{\max}\cH_q$.
Combining these inequalities proves equation~\eqref{eq:massbounds}. If $|s_k-1|\leq\delta<1$, the looser bounds obtained by replacing $\rho,s_{\min},s_{\max}$ with $\delta,1-\delta,1+\delta$, respectively, also hold.
\end{proof}

Minimum-size repair preserves pooled bin totals and the normalization identity while modifying the Dirichlet--multinomial count law. The theorem describes the allocation before repair; experiments measure realized histograms and transport after repair.

\section*{Supplementary Note 3. Finite stratified IID allocation}
\addcontentsline{toc}{section}{Supplementary Note 3. Finite stratified IID allocation}
\begin{lemma}[Rounding error of the stratified reference]\label{lem:iid}
Suppose the reference has $m_{kb}\in\{\lfloor N_b/K\rfloor,\lceil N_b/K\rceil\}$ and $N_k\in\{\lfloor N/K\rfloor,\lceil N/K\rceil\}$, with $N_{\min}=\min_k N_k>0$. Then
\begin{equation}
\|\vect{q}_k-\vect{\pi}\|_2\leq\frac{\sqrt B+\|\vect{\pi}\|_2}{N_{\min}},
\quad
\cH_q\leq\frac{(\sqrt B+\|\vect{\pi}\|_2)^2}{N_{\min}^2}.
\end{equation}
If $K$ divides $N$, the numerator $\sqrt B+\|\vect{\pi}\|$ can be replaced by $\sqrt B$. If $K$ divides every $N_b$, all client histograms equal $\vect{\pi}$ exactly.
\end{lemma}
\begin{proof}
Write $m_{kb}=N_b/K+e_{kb}$ and $N_k=N/K+d_k$. The rounding conditions imply $|e_{kb}|<1$ and $|d_k|<1$. Direct subtraction gives
\begin{equation}
\vect{q}_k-\vect{\pi}=(\vect{e}_k-d_k\vect{\pi})/N_k.
\end{equation}
The triangle inequality and $\|\vect{e}_k\|_2\leq\sqrt B$ yield the first bound. Squaring and averaging with weights $p_k$ gives the second. Divisibility of $N$ gives $d_k=0$; divisibility of all bin counts additionally gives $\vect{e}_k=\zeros$.
\end{proof}
The implementation first assigns $\lfloor N_b/K\rfloor$ members of each bin to every client, then assigns its remainder to distinct clients with the greatest unused capacities. Capacities enforce the floor/ceiling total sizes. Stratification controls bin-composition and size differences. The disjoint client datasets retain finite within-bin variation, which can produce gradient differences even with identical histograms.

\section*{Supplementary Note 4. Transport geometry and discretization error}
\addcontentsline{toc}{section}{Supplementary Note 4. Transport geometry and discretization error}
The implementation constructs response bins by $k$-means with $k$-means++ seeding \cite{arthur2007kmeans} and compares their histograms using discrete optimal transport \cite{peyre2019ot}. Let distinct centres $\vect{c}_1,\ldots,\vect{c}_B$ define $C_{bc}=\|\vect{c}_b-\vect{c}_c\|$. Write $c_{\min}=\min_{b\ne c}C_{bc}>0$ and $c_{\max}=\max_{b,c}C_{bc}$. For histograms $\vect{q},\vect{r}$, $\TV(\vect{q},\vect{r})=\|\vect{q}-\vect{r}\|_1/2$.

\begin{proposition}[Metric comparison and quantization]\label{prop:transport}
The discrete transport metric satisfies
\begin{equation}\label{eq:tv}
c_{\min}\TV(\vect{q},\vect{r})\leq W_{\mat{C}}(\vect{q},\vect{r})\leq c_{\max}\TV(\vect{q},\vect{r}),
\end{equation}
and hence
\begin{equation}\label{eq:l2transport}
\frac{c_{\min}}2\|\vect{q}-\vect{r}\|_2\leq W_{\mat{C}}(\vect{q},\vect{r})
\leq\frac{c_{\max}\sqrt B}2\|\vect{q}-\vect{r}\|_2.
\end{equation}
If $\mu_i$ is the empirical distribution of normalized responses on client $i$ and $\varepsilon_i$ is its mean distance to assigned centres, then
\begin{equation}\label{eq:quantization}
\left|W_1(\mu_i,\mu_j)-W_{\mat{C}}(\vect{q}_i,\vect{q}_j)\right|\leq\varepsilon_i+\varepsilon_j.
\end{equation}
\end{proposition}
\begin{proof}
Any coupling must move at least $\TV(\vect{q},\vect{r})$ off the diagonal because at most $\sum_b\min(q_b,r_b)=1-\TV(\vect{q},\vect{r})$ can remain on the same centre. Each off-diagonal unit costs at least $c_{\min}$, proving the lower bound. For the upper bound, keep the common mass on the diagonal and transport only the remaining $\TV(\vect{q},\vect{r})$; its cost per unit is at most $c_{\max}$. The norm inequalities $\|\vect{x}\|_2\leq\|\vect{x}\|_1\leq\sqrt B\|\vect{x}\|_2$, the latter by Cauchy--Schwarz, give equation~\eqref{eq:l2transport}.

Let $\widehat\mu_i=\sum_bq_{ib}\delta_{\vect{c}_b}$. Assigning each response to its bin centre is a coupling of cost $\varepsilon_i$, so $W_1(\mu_i,\widehat\mu_i)\leq\varepsilon_i$. The triangle inequality gives
$W_1(\mu_i,\mu_j)\leq\varepsilon_i+W_1(\widehat\mu_i,\widehat\mu_j)+\varepsilon_j$.
Applying it in the opposite direction proves equation~\eqref{eq:quantization}, since transport between the two centroid measures is $W_{\mat{C}}(\vect{q}_i,\vect{q}_j)$.
\end{proof}

Distinct centres supply the positive lower equivalence constant. With coincident centres, the upper bounds remain valid and the lower constant becomes zero. For sample weights $p_k$ and $\vect{\pi}=\sum_kp_k\vect{q}_k$, equation~\eqref{eq:l2transport} implies
\begin{equation}
\frac{c_{\min}^2}{4}\cH_q\leq\sum_kp_kW_{\mat{C}}(\vect{q}_k,\vect{\pi})^2
\leq\frac{Bc_{\max}^2}{4}\cH_q.
\end{equation}
The weighted variance identity
$\sum_{i,j}p_ip_j\|\vect{q}_i-\vect{q}_j\|^2=2\cH_q$
also gives
\begin{equation}
\frac{c_{\min}^2}{2}\cH_q\leq\sum_{i,j}p_ip_jW_{\mat{C}}(\vect{q}_i,\vect{q}_j)^2
\leq\frac{Bc_{\max}^2}{2}\cH_q.
\end{equation}
These are weighted squared quantities. The plotted $D_{\mathrm{sol}}$ is the unweighted mean of pairwise distances. If $S_W$ is the corresponding unweighted mean of squared pairwise distances, Jensen's inequality and $W_{\mat{C}}\leq c_{\max}$ give $D_{\mathrm{sol}}^2\leq S_W\leq c_{\max}D_{\mathrm{sol}}$. These identities specify the averaging operations needed to relate allocation moments to the plotted diagnostic.

\section*{Supplementary Note 5. Bin imbalance and gradient disagreement}
\addcontentsline{toc}{section}{Supplementary Note 5. Bin imbalance and gradient disagreement}
Federated convergence analyses use gradient dissimilarity to quantify heterogeneity \cite{li2020fedprox}, although large dissimilarity can coexist with effective FedAvg convergence \cite{wang2024effectiveness}. Here we separate bin-mixture and within-bin contributions. At a fixed parameter vector $\vect{\theta}\in\R^{d_\theta}$, let $F_b^{\mathrm{bin}}$ be the loss conditional on bin $b$, $\vect{\mu}_b=\nabla F_b^{\mathrm{bin}}(\vect{\theta})$ and $\mat{G}=[\vect{\mu}_1,\ldots,\vect{\mu}_B]\in\R^{d_\theta\times B}$. Define the zero-sum contrast space $\mathcal V=\{\vect{v}\in\R^B:\ones_B^\top \vect{v}=0\}$ and the restricted singular values
\begin{equation}
\gamma_- =\inf_{\vect{v}\in\mathcal V,\|\vect{v}\|=1}\|\mat{G}\vect{v}\|,
\quad \gamma_+ =\sup_{\vect{v}\in\mathcal V,\|\vect{v}\|=1}\|\mat{G}\vect{v}\|.
\end{equation}
\begin{proposition}[Population mixture and empirical residual]\label{prop:gradient}
For client population objectives $F_k=\sum_bq_{kb}F_b^{\mathrm{bin}}$, let $p_k>0$, $\sum_kp_k=1$, $\vect{\pi}=\sum_kp_k\vect{q}_k$, $F=\sum_kp_kF_k$, $\vect{g}_k=\nabla F_k$ and $\vect{g}=\nabla F$. Then
\begin{equation}\label{eq:gradbounds}
\gamma_-^2\cH_q\leq\cH_g\leq\gamma_+^2\cH_q,
\quad \cH_g=\sum_kp_k\|\nabla F_k-\nabla F\|^2.
\end{equation}
For an exhaustive finite partition with sample weights $p_k=N_k/N$, define $\vect{\mu}_b$ as the mean gradient of all pooled examples in bin $b$. Let $\vect{g}_k$ be the empirical client gradient, $\vect{\epsilon}_k=\vect{g}_k-\mat{G}\vect{q}_k$ and $\sigma_{\mathrm{bin}}^2=\sum_kp_k\|\vect{\epsilon}_k\|^2$. Then $\sum_kp_k\vect{\epsilon}_k=\zeros$ and
\begin{equation}\label{eq:empiricalgrad}
(\gamma_-\sqrt{\cH_q}-\sigma_{\mathrm{bin}})_+
\leq\sqrt{\cH_g}
\leq\gamma_+\sqrt{\cH_q}+\sigma_{\mathrm{bin}}.
\end{equation}
\end{proposition}
\begin{proof}
For the population mixture, $\vect{g}_k-\vect{g}=\mat{G}(\vect{q}_k-\vect{\pi})$ and $\vect{q}_k-\vect{\pi}\in\mathcal V$. The defining bounds on $\mat{G}$ apply to each client; multiplying their squares by $p_k$ and summing proves equation~\eqref{eq:gradbounds}.

For finite samples, $\sum_kp_k\vect{g}_k$ is the pooled empirical gradient, and $\sum_kp_k\mat{G}\vect{q}_k=\mat{G}\vect{\pi}$ is the same gradient grouped by bins. Hence the residuals average to zero. In the weighted product norm $\|\vect{a}\|_p^2=\sum_kp_k\|\vect{a}_k\|^2$, apply the triangle and reverse triangle inequalities to
$\vect{g}_k-\vect{g}=\mat{G}(\vect{q}_k-\vect{\pi})+\vect{\epsilon}_k$,
then apply the bounds defining $\gamma_-$ and $\gamma_+$. This gives equation~\eqref{eq:empiricalgrad}.
\end{proof}
A positive $\gamma_-$ requires the model and loss to distinguish every nonzero bin contrast at the evaluated parameter vector. Response-centroid separation describes data geometry; $\gamma_-$ describes gradient geometry. The plotted diagnostic normalizes $\cH_g$ by $\max\{\sum_kp_k\|\vect{g}_k\|^2,10^{-12}\}$, as defined in Methods.

\section*{Supplementary Note 6. Bin-conditioned gradient correction}
\addcontentsline{toc}{section}{Supplementary Note 6. Bin-conditioned gradient correction}
Control-variate methods correct client drift in federated optimization \cite{karimireddy2020scaffold}. SB-CFO constructs its correction from bin-conditioned gradient prototypes. All gradients in this note are evaluated at one fixed global parameter vector $\vect{\theta}$. Let $\vect{q}_k$ be the realized client histogram, $p_k=N_k/N$, $\vect{\pi}=\sum_kp_k\vect{q}_k$, and $\mat{G}=[\vect{\mu}_1,\ldots,\vect{\mu}_B]$, where $\vect{\mu}_b$ is the mean gradient over all training examples in bin $b$. As in Supplementary Proposition~\ref{prop:gradient}, write the empirical client gradient as $\vect{g}_k=\mat{G}\vect{q}_k+\vect{\epsilon}_k$, with $\sum_kp_k\vect{\epsilon}_k=\zeros$. The pooled gradient is $\vect{g}=\mat{G}\vect{\pi}$.

For an occupied client--bin pair, let $\widehat{\vect{\mu}}_{kb}$ be its mean gradient, computed locally. The server constructs
\begin{equation}\label{eq:prototype}
\widehat{\vect{\mu}}_b=\sum_{k:m_{kb}>0}\frac{m_{kb}}{N_b}\widehat{\vect{\mu}}_{kb},\quad
\widehat{\mat{G}}=[\widehat{\vect{\mu}}_1,\ldots,\widehat{\vect{\mu}}_B],\quad
\vect{r}_k=\widehat{\mat{G}}(\vect{\pi}-\vect{q}_k).
\end{equation}
The weights use the full client--bin counts, including when a prototype is estimated from a uniform subsample. Full-bin evaluation at $\vect{\theta}$ gives $\widehat{\mat{G}}=\mat{G}$. Empty local bins contribute zero; every global bin is nonempty. The correction is held fixed during the following local steps.

\begin{proposition}[Exact centering and residual heterogeneity]\label{prop:correction}
Let $\vect{a}_k=\vect{g}_k+\vect{r}_k$, $\cH_a=\sum_kp_k\|\vect{a}_k-\vect{g}\|^2$, and $\sigma_{\mathrm{bin}}^2=\sum_kp_k\|\vect{\epsilon}_k\|^2$. Define
$\tau=\sup_{\vect{v}\in\mathcal V,\|\vect{v}\|=1}\|(\mat{G}-\widehat{\mat{G}})\vect{v}\|$ on $\mathcal V=\{\vect{v}\in\R^B:\ones_B^\top \vect{v}=0\}$. Then
\begin{align}
\sum_kp_k\vect{r}_k&=\zeros,\quad \sum_kp_k\vect{a}_k=\vect{g},\label{eq:zero_correction}\\
\vect{a}_k-\vect{g}&=\vect{\epsilon}_k+(\mat{G}-\widehat{\mat{G}})(\vect{q}_k-\vect{\pi}),\label{eq:corrected_residual}\\
\sqrt{\cH_a}&\leq\sigma_{\mathrm{bin}}+\tau\sqrt{\cH_q}.\label{eq:correction_bound}
\end{align}
In particular, fresh full-data prototypes give $\cH_a=\sigma_{\mathrm{bin}}^2$. A sufficient condition for a strict reduction of gradient variance is
$(\gamma_--\tau)\sqrt{\cH_q}>2\sigma_{\mathrm{bin}}$.
\end{proposition}
\begin{proof}
The weighted correction is $\widehat{\mat{G}}(\vect{\pi}-\sum_kp_k\vect{q}_k)=\zeros$. Substitution of $\vect{g}_k=\mat{G}\vect{q}_k+\vect{\epsilon}_k$ gives equation~\eqref{eq:corrected_residual}. Apply the triangle inequality in the weighted product norm $\|\vect{z}\|_p^2=\sum_kp_k\|\vect{z}_k\|^2$, followed by the definition of the restricted operator norm $\tau$, to obtain equation~\eqref{eq:correction_bound}. If $\widehat{\mat{G}}=\mat{G}$, only $\vect{\epsilon}_k$ remains. Finally, Supplementary Proposition~\ref{prop:gradient} gives $\sqrt{\cH_g}\geq(\gamma_-\sqrt{\cH_q}-\sigma_{\mathrm{bin}})_+$. The stated sufficient condition makes this lower bound strictly larger than the upper bound for $\sqrt{\cH_a}$.
\end{proof}
The sufficient condition accounts for possible cancellation between the mixture component and within-bin residuals before correction.

If prototypes were computed at $\vect{\theta}_{t-d}$ and each bin gradient is $L_b$-Lipschitz, then the triangle inequality and the bound of the operator norm by the Frobenius norm imply
\begin{equation}\label{eq:prototype_age}
\tau_t\leq\left(\sum_b L_b^2\right)^{1/2}\|\vect{\theta}_t-\vect{\theta}_{t-d}\|
+\|\mat{G}(\vect{\theta}_{t-d})-\widehat{\mat{G}}\|_{\mathrm{op}}.
\end{equation}
The second term is prototype estimation error. In the reported training runs, prototypes use the full training set at the start of every round. At the final checkpoint, diagnostics evaluate the correction used for the preceding round, so its age is one round. Equation~\eqref{eq:zero_correction} holds for either age. For a fixed correction and any trial parameter $\vect{\vartheta}\in\R^{d_\theta}$, $\sum_kp_k[F_k(\vect{\vartheta})+\langle \vect{r}_k,\vect{\vartheta}\rangle]=F(\vect{\vartheta})$; the pooled loss is preserved. Its adaptive optimization trajectory additionally depends on the Adam states.

\section*{Supplementary Note 7. Stateful local Adam and update dispersion}
\addcontentsline{toc}{section}{Supplementary Note 7. Stateful local Adam and update dispersion}
We analyse the local Adam update used in the implementation, including persistent moments and bias correction. Real and imaginary FNO parameters are represented by separate real coordinates. Within one round, write $\vect{\theta}_{k,0}=\vect{\theta}$, and let $n$ be the number of completed local optimizer steps. Full participation and a fixed local-step count make $n$ common to clients. Initial states $\vect{m}_{k,0},\vect{v}_{k,0}$ are client-specific and are retained from previous rounds. For $e=0,\ldots,E-1$, define
\begin{align}
\widetilde{\vect{g}}_{k,e}&=\nabla F_k(\vect{\theta}_{k,e})+\vect{r}_k+\vect{\xi}_{k,e},\nonumber\\
\vect{m}_{k,e+1}&=\beta_1\vect{m}_{k,e}+(1-\beta_1)\widetilde{\vect{g}}_{k,e},\nonumber\\
\vect{v}_{k,e+1}&=\beta_2\vect{v}_{k,e}+(1-\beta_2)\widetilde{\vect{g}}_{k,e}^{\odot2},\nonumber\\
\widehat{\vect{m}}_{k,e+1}&=\frac{\vect{m}_{k,e+1}}{1-\beta_1^{n+e+1}},\quad
\widehat{\vect{v}}_{k,e+1}=\frac{\vect{v}_{k,e+1}}{1-\beta_2^{n+e+1}},\nonumber\\
\mat{D}_{k,e}&=\operatorname{diag}\left[(\sqrt{\widehat{\vect{v}}_{k,e+1}}+\epsilon_{\mathrm A}\ones_{d_\theta})^{-1}\right],\nonumber\\
\vect{\theta}_{k,e+1}&=\vect{\theta}_{k,e}-\eta \mat{D}_{k,e}\widehat{\vect{m}}_{k,e+1}.\label{eq:adam_exact}
\end{align}
The square, square root and reciprocal are coordinatewise. FedAvg uses $\vect{r}_k=\zeros$; SB-CFO uses equation~\eqref{eq:prototype}. Parameters are averaged after $E$ steps, while moment states remain local.

\paragraph{Conditional assumptions.}
Condition on the data, partition, starting model, prototype estimates and moment states at the start of the round. Expectations in this note are over subsequent minibatch draws under that conditioning. Assume $E\geq1$ and $n\geq0$ are integers, $\eta>0$, $0\leq\beta_1,\beta_2<1$, $\epsilon_{\mathrm A}>0$, $\vect{v}_{k,0}\geq\zeros$, and: (i) each $\nabla F_k$ is $L$-Lipschitz on a region containing the visited parameters and the segments joining them to $\vect{\theta}$; (ii) corrected sampled gradients have norm at most $M$, and $\|\vect{m}_{k,0}\|\leq(1-\beta_1^n)M$; (iii) $\E[\vect{\xi}_{k,e}\mid\mathcal F_e]=0$ and $\E[\|\vect{\xi}_{k,e}\|^2\mid\mathcal F_e]\leq\sigma_k^2$, where $\mathcal F_e$ contains the history before step $e$; (iv) there is a starting-history-measurable diagonal matrix $\mat{D}$ with
\begin{equation}\label{eq:adam_assumptions}
0<\underline d \mat{I}\preceq \mat{D}\preceq\overline d \mat{I},\quad
\|\mat{D}_{k,e}\|_{\mathrm{op}}\leq\overline d,\quad
\|\mat{D}_{k,e}-\mat{D}\|_{\mathrm{op}}\leq\delta_D
\end{equation}
almost surely. The scalar $\eta$ is constant within the round, as in the round-indexed cosine schedule. The initial-moment bound in (ii) follows if moments start at zero and all historical corrected sampled gradients are bounded by the same $M$; uniform minibatch sampling within each fixed client gives the conditional unbiasedness in (iii). The adaptive direction retains dependence on the current gradient through its second moment. We use $\beta_1^0=1$, including when $\beta_1=0$.

The denominator floor supplies the finite bound $\|\mat{D}_{k,e}\|_{\mathrm{op}}\leq\epsilon_{\mathrm A}^{-1}$. Informative dispersion bounds additionally require control of $\delta_D$ for the persistent, client-specific states. For ReLU networks, condition (i) requires Lipschitz gradients on the stated region, including any crossed activation boundaries. These explicit smoothness, bounded-gradient and effective-step assumptions follow the conditional approach of Adam analyses \cite{reddi2018adam,chen2019adamtype,defossez2022simple}.

Let $\vect{p}=(p_1,\ldots,p_K)^\top$ and stack client vectors as $\vect{h}=\operatorname{col}(\vect{h}_1,\ldots,\vect{h}_K)\in\R^{Kd_\theta}$, with $\|\vect{h}\|_p^2=\sum_kp_k\|\vect{h}_k\|^2$. The centering matrix is
\begin{equation}
\mat{P}=(\mat{I}_K-\ones_K\vect{p}^{\top})\otimes\mat{I}_{d_\theta},
\quad
(\mat{P}\vect{h})_k=\vect{h}_k-\sum_jp_j\vect{h}_j,
\end{equation}
where $\otimes$ denotes the Kronecker product. Expanding the square gives
\begin{equation}\label{eq:projection}
\|\mat{P}\vect{h}\|_p^2=\sum_kp_k\|\vect{h}_k\|^2-\left\|\sum_kp_k\vect{h}_k\right\|^2\leq\|\vect{h}\|_p^2.
\end{equation}
Define $\cH_{Da}=\sum_kp_k\|\mat{D}(\vect{a}_k-\vect{g})\|^2$, $\cH_{Dm}=\sum_kp_k\|\mat{D}(\vect{m}_{k,0}-\overline{\vect{m}}_0)\|^2$, $\overline{\vect{m}}_0=\sum_kp_k\vect{m}_{k,0}$ and $\sigma_p^2=\sum_kp_k\sigma_k^2$. Let
\begin{align}
a_e&=\frac{\beta_1^{e+1}}{1-\beta_1^{n+e+1}},&
b_e&=\frac{1-\beta_1^{e+1}}{1-\beta_1^{n+e+1}},\nonumber\\
w_{ej}&=\frac{(1-\beta_1)\beta_1^{e-j}}{1-\beta_1^{n+e+1}}\quad(0\leq j\leq e),&
A_E&=\sum_{e=0}^{E-1}a_e,\quad B_E=\sum_{e=0}^{E-1}b_e.\label{eq:adam_weights}
\end{align}
Direct summation gives $\sum_{j=0}^ew_{ej}=b_e\leq1$ and $B_E+(1-\beta_1^n)A_E=E$.

\begin{theorem}[Local Adam dispersion with explicit state and noise terms]\label{thm:local}
Under the conditional assumptions, the endpoint admits
\begin{equation}\label{eq:adam_endpoint}
\vect{\theta}_{k,E}=\vect{\theta}-\eta \mat{D}(A_E\vect{m}_{k,0}+B_E\vect{a}_k)+\vect{\zeta}_k,
\end{equation}
where $\vect{\zeta}=\operatorname{col}(\vect{\zeta}_1,\ldots,\vect{\zeta}_K)$ satisfies $(\E\|\vect{\zeta}\|_p^2)^{1/2}\leq\mathcal R$, with
\begin{equation}\label{eq:adam_remainder}
\mathcal R=\eta E\delta_D M+
\eta^2\overline d^{\,2}LM\frac{E(E-1)}2+
\eta\overline d\sigma_p\sum_{e=0}^{E-1}\left(\sum_{j=0}^ew_{ej}^2\right)^{1/2}.
\end{equation}
Consequently, for $D_{\mathrm{loc}}=\sum_kp_k\|\vect{\theta}_{k,E}-\sum_jp_j\vect{\theta}_{j,E}\|^2$,
\begin{equation}\label{eq:adam_dispersion}
\left|\sqrt{\E D_{\mathrm{loc}}}-\eta B_E\sqrt{\cH_{Da}}\right|
\leq\eta A_E\sqrt{\cH_{Dm}}+\mathcal R.
\end{equation}
\end{theorem}
\begin{proof}
Unrolling the first-moment recurrence and applying its bias correction yields
\begin{equation}
\widehat{\vect{m}}_{k,e+1}=a_e\vect{m}_{k,0}+b_e\vect{a}_k+
\sum_{j=0}^ew_{ej}[\nabla F_k(\vect{\theta}_{k,j})-\vect{g}_k]+\sum_{j=0}^ew_{ej}\vect{\xi}_{k,j}.
\end{equation}
Its norm is at most $M$: the historical and new-gradient weights have total mass $a_e(1-\beta_1^n)+b_e=1$. Thus, the triangle inequality in equation~\eqref{eq:adam_exact} gives $\|\vect{\theta}_{k,j}-\vect{\theta}\|\leq\eta j\overline d M$. Substitute $\mat{D}_{k,e}=\mat{D}+(\mat{D}_{k,e}-\mat{D})$ into the sum of updates. The constant-gradient and historical-moment terms give equation~\eqref{eq:adam_endpoint}. By submultiplicativity and the triangle inequality, the preconditioner remainder has norm at most $\eta E\delta_D M$ for every client. The path remainder is bounded by
\begin{equation}
\eta\overline d\sum_{e=0}^{E-1}\sum_{j=0}^ew_{ej}L\eta j\overline d M
\leq\eta^2\overline d^{\,2}LM\sum_{e=0}^{E-1}e,
\end{equation}
using $w_{ej}\geq0$ and $\sum_jw_{ej}\leq1$.

For $i<j$, the tower property and conditional unbiasedness give
$\E\langle\vect{\xi}_{k,i},\vect{\xi}_{k,j}\rangle=\E\langle\vect{\xi}_{k,i},\E[\vect{\xi}_{k,j}\mid\mathcal F_j]\rangle=0$.
Therefore $\E\|\sum_jw_{ej}\vect{\xi}_{k,j}\|^2\leq\sigma_k^2\sum_jw_{ej}^2$. Minkowski's inequality across local steps in the square-integrable weighted product space gives the final term in equation~\eqref{eq:adam_remainder}. Dependence between the current stochastic gradient and $\mat{D}_{k,e}$ is retained in the preconditioner remainder. Finally center equation~\eqref{eq:adam_endpoint}, apply the reverse triangle inequality in that same product space, and use equation~\eqref{eq:projection}. The centered historical-moment term has norm $\eta A_E\sqrt{\cH_{Dm}}$, proving equation~\eqref{eq:adam_dispersion}.
\end{proof}

\begin{corollary}[Bin imbalance and corrected Adam dispersion]\label{cor:adam_bins}
Write $\mathcal B=\eta A_E\sqrt{\cH_{Dm}}+\mathcal R$. For uncorrected local Adam,
\begin{align}
\sqrt{\E D_{\mathrm{loc}}}&\geq
\left[\eta B_E\underline d(\gamma_-\sqrt{\cH_q}-\sigma_{\mathrm{bin}})_+-\mathcal B\right]_+,\nonumber\\
\sqrt{\E D_{\mathrm{loc}}}&\leq
\eta B_E\overline d(\gamma_+\sqrt{\cH_q}+\sigma_{\mathrm{bin}})+\mathcal B.\label{eq:adam_bin_bounds}
\end{align}
For SB-CFO, with the constants and states of its own trajectory,
\begin{equation}\label{eq:sbcfo_adam_bound}
\sqrt{\E D_{\mathrm{loc}}}\leq
\eta B_E\overline d(\sigma_{\mathrm{bin}}+\tau\sqrt{\cH_q})+\mathcal B.
\end{equation}
\end{corollary}
\begin{proof}
The eigenvalue bounds on $\mat{D}$ imply $\underline d\sqrt{\cH_a}\leq\sqrt{\cH_{Da}}\leq\overline d\sqrt{\cH_a}$. Apply Supplementary Proposition~\ref{prop:gradient} for uncorrected gradients, Proposition~\ref{prop:correction} for corrected gradients, and Theorem~\ref{thm:local} for each algorithm.
\end{proof}
The lower bound in equation~\eqref{eq:adam_bin_bounds} is positive only when
\begin{equation}\label{eq:adam_positive_condition}
\eta B_E\underline d\bigl(\gamma_-\sqrt{\cH_q}-\sigma_{\mathrm{bin}}\bigr)_+>\mathcal B.
\end{equation}
The bin-mixture component must therefore exceed the historical-moment and remainder terms. Fresh exact prototypes set $\tau=0$ in equation~\eqref{eq:sbcfo_adam_bound}, removing its explicit mixture term. A sufficient comparison places the state-specific SB-CFO upper bound below the corresponding uncorrected lower bound.

The allocation laws concern unrepaired raw masses in expectation, whereas these update bounds condition on a realized normalized partition. Minimum-size repair, within-bin residuals and adaptive states can all change with $\alpha$. A monotone update or prediction law would additionally require control of how these state-dependent terms change with $\alpha$.

\section*{Supplementary Note 8. Adam aggregation and a conditional stationarity bound}
\addcontentsline{toc}{section}{Supplementary Note 8. Adam aggregation and a conditional stationarity bound}
An exact algebraic identity describes why adaptive updates need not cancel even when raw gradients do. At any local step, define $\overline{\mat{D}}=\sum_kp_k\mat{D}_{k,e}$ and $\overline{\widehat{\vect{m}}}=\sum_kp_k\widehat{\vect{m}}_{k,e+1}$. Then
\begin{equation}\label{eq:adam_covariance}
\sum_kp_k\mat{D}_{k,e}\widehat{\vect{m}}_{k,e+1}
=\overline{\mat{D}}\,\overline{\widehat{\vect{m}}}
+\sum_kp_k(\mat{D}_{k,e}-\overline{\mat{D}})(\widehat{\vect{m}}_{k,e+1}-\overline{\widehat{\vect{m}}}).
\end{equation}
To verify the identity, expand the product on the right and use the zero weighted means of both centered factors. The operator norm inequality and weighted Cauchy--Schwarz bound the final term by
\begin{equation}
\left(\sum_kp_k\|\mat{D}_{k,e}-\overline{\mat{D}}\|_{\mathrm{op}}^2\right)^{1/2}
\left(\sum_kp_k\|\widehat{\vect{m}}_{k,e+1}-\overline{\widehat{\vect{m}}}\|^2\right)^{1/2}.
\end{equation}
Both the mean moment and this covariance depend on previous corrected gradients. Thus, zero-mean corrections preserve the pooled objective while the adaptive update also contains a preconditioner--moment covariance.

For a fixed round, sample-weighted aggregation of equation~\eqref{eq:adam_endpoint} gives
\begin{equation}\label{eq:adam_round}
\vect{\theta}_{t+1}=\vect{\theta}_t-\eta_t\mat{D}_t(A_E\vect{m}_{\mathrm{av},t}+B_E\vect{g}_t)+\overline{\vect{\zeta}}_t,
\quad \vect{g}_t=\nabla F(\vect{\theta}_t).
\end{equation}
Here $\vect{m}_{\mathrm{av},t}=\sum_kp_k\vect{m}_{k,0}$ and $\overline{\vect{\zeta}}_t=\sum_kp_k\vect{\zeta}_{k,t}$. The weights $A_E,B_E$ are evaluated at the current completed-step count $n_t$. The identity holds for either FedAvg or SB-CFO because $\sum_kp_k\vect{r}_k=\zeros$. Using $B_E+(1-\beta_1^{n_t})A_E=E$, set $h_t=\eta_tE$ and
\begin{equation}\label{eq:adam_error}
\vect{e}_t=\frac{A_E}{E}\mat{D}_t[\vect{m}_{\mathrm{av},t}-(1-\beta_1^{n_t})\vect{g}_t]
-\frac{\overline{\vect{\zeta}}_t}{\eta_tE}.
\end{equation}
Then $\vect{\theta}_{t+1}=\vect{\theta}_t-h_t(\mat{D}_t\vect{g}_t+\vect{e}_t)$ exactly. The first term in $\vect{e}_t$ is the lag of the pooled first moment; the second contains preconditioner variation, local path effects and stochastic gradients. Convexity of the squared norm gives $\E\|\overline{\vect{\zeta}}_t\|^2\leq\E\|\vect{\zeta}_t\|_p^2$. Thus its conditional root mean square is at most $\mathcal R_t$, and the triangle inequality bounds that of $\vect{e}_t$ by
\begin{equation}\label{eq:state_error_bound}
\frac{A_E}{E}\overline d\|\vect{m}_{\mathrm{av},t}-(1-\beta_1^{n_t})\vect{g}_t\|
+\frac{\mathcal R_t}{\eta_tE}.
\end{equation}

\begin{theorem}[Finite-horizon stationarity with state error]\label{thm:stationarity}
Fix the initial parameter $\vect{\theta}_0$. Suppose $F\geq F_*$ and $\nabla F$ is $L_F$-Lipschitz, with $L_F>0$, on all global update segments. Assume fixed bounds $0<\underline d \mat{I}\preceq \mat{D}_t\preceq\overline d \mat{I}$, deterministic $h_t>0$, finite second moments of $\vect{e}_t$, and $h_t\leq\underline d/(4L_F\overline d^{\,2})$. Then
\begin{equation}\label{eq:stationarity}
\frac{\sum_{t<T}h_t\E\|\vect{g}_t\|^2}{\sum_{t<T}h_t}
\leq\frac{2(F(\vect{\theta}_0)-F_*)}{\underline d\sum_{t<T}h_t}
+\frac{2}{\underline d}\left(\frac1{\underline d}+L_Fh_{\max}\right)
\frac{\sum_{t<T}h_t\E\|\vect{e}_t\|^2}{\sum_{t<T}h_t},
\end{equation}
where $h_{\max}=\max_{t<T}h_t$.
\end{theorem}
\begin{proof}
The descent lemma gives
\begin{align}
F(\vect{\theta}_{t+1})
&\leq F(\vect{\theta}_t)-h_t\langle \vect{g}_t,\mat{D}_t\vect{g}_t+\vect{e}_t\rangle
+\frac{L_Fh_t^2}{2}\|\mat{D}_t\vect{g}_t+\vect{e}_t\|^2\nonumber\\
&\leq F(\vect{\theta}_t)-h_t\underline d\|\vect{g}_t\|^2
+h_t\|\vect{g}_t\|\|\vect{e}_t\|
+L_Fh_t^2(\overline d^{\,2}\|\vect{g}_t\|^2+\|\vect{e}_t\|^2).
\end{align}
The second line uses Cauchy--Schwarz, the eigenvalue bounds and $\|\vect{x}+\vect{y}\|^2\leq2\|\vect{x}\|^2+2\|\vect{y}\|^2$. Young's inequality in the form $ab\leq\underline d a^2/4+b^2/\underline d$, followed by the step-size assumption, yields
\begin{equation}
F(\vect{\theta}_{t+1})\leq F(\vect{\theta}_t)-\frac{h_t\underline d}{2}\|\vect{g}_t\|^2
+h_t(\underline d^{-1}+L_Fh_t)\|\vect{e}_t\|^2.
\end{equation}
Take expectations, telescope from $0$ to $T-1$, use $F(\vect{\theta}_T)\geq F_*$, and divide by $\underline d\sum_th_t/2$.
\end{proof}
The theorem bounds descent through the exact perturbed-preconditioned representation of an Adam round. The mean-square term includes bias and dependence of $\vect{e}_t$ on both $\mat{D}_t$ and $\vect{g}_t$. If this magnitude stays bounded away from zero, the bound retains a nonvanishing error term. With the stated uniform step-size restriction, diverging $\sum_th_t$ and vanishing weighted mean-square error suffice for the weighted mean squared gradient norm to vanish.

Adam convergence depends on effective step sizes and optimizer states \cite{reddi2018adam,chen2019adamtype,defossez2022simple}. SB-CFO controls the bin-mixture contribution in Note~7, while moment lag, preconditioner variation and sampling noise enter the state-error bound in equation~\eqref{eq:state_error_bound}.

\section*{Supplementary Note 9. Parameter separation, prediction error and an Adam example}
\addcontentsline{toc}{section}{Supplementary Note 9. Parameter separation, prediction error and an Adam example}
A trajectory comparison with persistent moments uses the joint parameter--moment state. Let $\vect{\chi}_t=(\vect{\theta}_t,\{\vect{m}_{k,t},\sqrt{\vect{v}_{k,t}}\}_{k=1}^K)$, equipped with the Euclidean product norm, which dominates the norm of its parameter component. For the evaluated full-participation scheme, the completed-step count is $n_t=tE$ and prototypes are refreshed every round, so this state specifies the Adam round map $\mathcal T_{P,t}$ for a fixed partition $P$ and fixed minibatch draws. Couple the two runs using a common source of sampling randomness, mapped to each partition. Conditional on this common randomness, if $\mathcal T_{\alpha,t}$ is $L_t$-Lipschitz on the compared states, define
$b_t=\|\mathcal T_{\alpha,t}(\vect{\chi}_{\mathrm I,t})-\mathcal T_{\mathrm I,t}(\vect{\chi}_{\mathrm I,t})\|$.
Adding and subtracting $\mathcal T_{\alpha,t}(\vect{\chi}_{\mathrm I,t})$ and applying the triangle inequality gives $d_{t+1}\leq L_td_t+b_t$, where $d_t=\|\vect{\chi}_{\alpha,t}-\vect{\chi}_{\mathrm I,t}\|$. Repeated substitution from identical initialization yields
\begin{equation}\label{eq:adam_accumulation}
\|\vect{\theta}_{\alpha,T}-\vect{\theta}_{\mathrm I,T}\|\leq d_T
\leq\sum_{t=0}^{T-1}b_t\prod_{j=t+1}^{T-1}L_j.
\end{equation}
On a bounded region containing the compared local trajectories, Lipschitz per-example gradients, bounded moment states and a positive denominator floor suffice for a Lipschitz round map when the minibatch draws are fixed. The square-root second-moment update is a weighted Euclidean norm and is Lipschitz in its input components. The products of $L_j$ quantify amplification or attenuation of earlier perturbations.

\begin{proposition}[Prediction sensitivity]\label{prop:risk}
For a differentiable test risk $R$ with $L_R$-Lipschitz gradient on the segment from $\vect{\theta}_{\mathrm I}$ to $\vect{\theta}_\alpha$, write $\Delta\vect{\theta}=\vect{\theta}_\alpha-\vect{\theta}_{\mathrm I}$. Then
\begin{equation}\label{eq:risktaylor}
\left|R(\vect{\theta}_\alpha)-R(\vect{\theta}_{\mathrm I})-\langle\nabla R(\vect{\theta}_{\mathrm I}),\Delta\vect{\theta}\rangle\right|
\leq\frac{L_R}{2}\|\Delta\vect{\theta}\|^2.
\end{equation}
If the test-set prediction map $\widehat{\mat{Y}}(\vect{\theta})$ is $L_Y$-Lipschitz from the Euclidean parameter norm to the Frobenius output norm and $\|\mat{Y}\|_F>0$, then
\begin{equation}\label{eq:errorlipschitz}
|e(\vect{\theta}_\alpha)-e(\vect{\theta}_{\mathrm I})|\leq\frac{L_Y}{\|\mat{Y}\|_F}\|\Delta\vect{\theta}\|.
\end{equation}
\end{proposition}
\begin{proof}
The fundamental theorem of calculus gives $R(\vect{\theta}_{\mathrm I}+\Delta\vect{\theta})-R(\vect{\theta}_{\mathrm I})=\int_0^1\langle\nabla R(\vect{\theta}_{\mathrm I}+s\Delta\vect{\theta}),\Delta\vect{\theta}\rangle\,ds$. Subtract the initial gradient term, use Cauchy--Schwarz and gradient Lipschitz continuity, and integrate $L_Rs\|\Delta\vect{\theta}\|^2$. For relative error, the reverse triangle inequality bounds the difference of residual norms by $\|\widehat{\mat{Y}}(\vect{\theta}_\alpha)-\widehat{\mat{Y}}(\vect{\theta}_{\mathrm I})\|_F$. Divide by $\|\mat{Y}\|_F$ and use the prediction-map Lipschitz bound.
\end{proof}
The signed risk change depends on the displacement direction and test-risk geometry. Parameter separation therefore controls a magnitude bound, while its predictive effect can have either sign.

\begin{proposition}[One-step Adam separation and exact bin correction]\label{prop:adam_witness}
Consider a one-dimensional parameter and two equally weighted clients at $\theta=0$ with losses $F_1(\theta)=(\theta-1)^2/2$ and $F_2(\theta)=(\theta+3)^2/2$. Their gradients are $-1$ and $3$, while a homogeneous reference has pooled gradient $1$. Starting from zero moment states, one local Adam step gives
\begin{equation}\label{eq:adam_witness}
\theta_{\mathrm{FA}}-\theta_{\mathrm I}
=\frac{3\eta}{(1+\epsilon_{\mathrm A})(3+\epsilon_{\mathrm A})}>0.
\end{equation}
If the two losses define two pure solution bins, fresh exact SB-CFO prototypes make both corrected gradients equal to $1$, and the corrected aggregate equals the homogeneous Adam update.
\end{proposition}
\begin{proof}
Bias correction at the first step gives $\widehat m=g$ and $\widehat v=g^2$ for any admissible $\beta_1,\beta_2$. Thus the scalar Adam direction is $f(g)=g/(|g|+\epsilon_{\mathrm A})$. Subtracting $-\eta f(1)$ from $-\eta[f(-1)+f(3)]/2$ and simplifying yields equation~\eqref{eq:adam_witness}. The prototype correction is $2$ on the first client and $-2$ on the second. Both corrected gradients are therefore $1$, giving identical first and second moments and identical updates. This exact one-step example identifies a mechanism; persistent unequal states and later local movement are covered by Note~7.
\end{proof}

\section*{Supplementary Note 10. Numerical verification and reference resolution}
\addcontentsline{toc}{section}{Supplementary Note 10. Numerical verification and reference resolution}
The code archive includes executable checks of the allocation moments, mass normalization, gradient-correction identities, Adam moment weights, local-dispersion bounds and finite-horizon descent inequality. Monte Carlo calculations test the allocation formulas, and bounded quadratic examples test the optimization identities.

\paragraph*{Burgers numerical-reference sensitivity.}
We compared eight fixed inputs from the seed-42 Burgers training pool (data seed 10084) on 100, 200 and 400 independent periodic spatial nodes. The input-only selection spans low, middle and high amplitude ranks and low and high endpoint-jump ranks, and includes the first training example and the largest endpoint jump. Refined grids use the piecewise-linear reconstruction of the original 101 input sensors. The comparison retains 101 stored times and projects each solution onto the original 101 spatial locations, including the duplicated periodic endpoint.

The baseline uses the full 800-input batch and its shared time-step rule. Spatial comparisons use the same Rusanov flux and a time-step safety factor of 0.175; the 400-node comparison reference uses 0.0875. Supplementary Table~\ref{tab:resolution} reports relative Frobenius differences after concatenating the eight response tensors. Initial conditions, arrays, scripts and per-example errors are supplied in Source Data.

\begin{table}[!htbp]
\caption{Burgers response sensitivity to spatial and temporal refinement. All entries are relative differences in percent. The original 100-node labels are compared with the 400-node reference at safety factor 0.0875. Adjacent-grid comparisons use 0.175 on both grids; the final column compares factors 0.175 and 0.0875 at 400 nodes.}\label{tab:resolution}
\centering\small
\begin{tabular}{lrrrr}
\toprule
$\nu$ & Original vs reference & 100 vs 200 & 200 vs 400 & Time refinement\\
\midrule
0.1 & 2.177 & 1.449 & 0.729 & 0.0006\\
0.05 & 2.678 & 1.782 & 0.906 & 0.0016\\
0.01 & 5.061 & 3.367 & 1.813 & 0.0165\\
0.005 & 6.553 & 4.330 & 2.466 & 0.0365\\
\bottomrule
\end{tabular}
\end{table}

Spatial differences decrease with refinement but remain appreciable at low viscosity. At $\nu=0.005$, the largest-jump input has a 16.75\% full-trajectory difference and a 27.44\% final-time difference from the finer reference. Because the sampled endpoints need not agree, grid refinement also changes the initial discrete mean, which the periodic conservative flux retains. The learning experiments use the original 100-node labels; the refinement results quantify their finite-grid sensitivity.
\FloatBarrier

\section*{Supplementary Figures}
\addcontentsline{toc}{section}{Supplementary Figures}
Unless stated otherwise, plotted means and intervals use five seeds. Confidence intervals are two-sided, pointwise 95\% Student-$t$ intervals. Error differences are paired within seed and expressed in percentage points. The final checkpoint is round 1,000. Public-task intervals condition on a fixed split and response codebook.
\begin{figure}[!htbp]
\centering\includegraphics[width=.95\linewidth,height=.68\textheight,keepaspectratio]{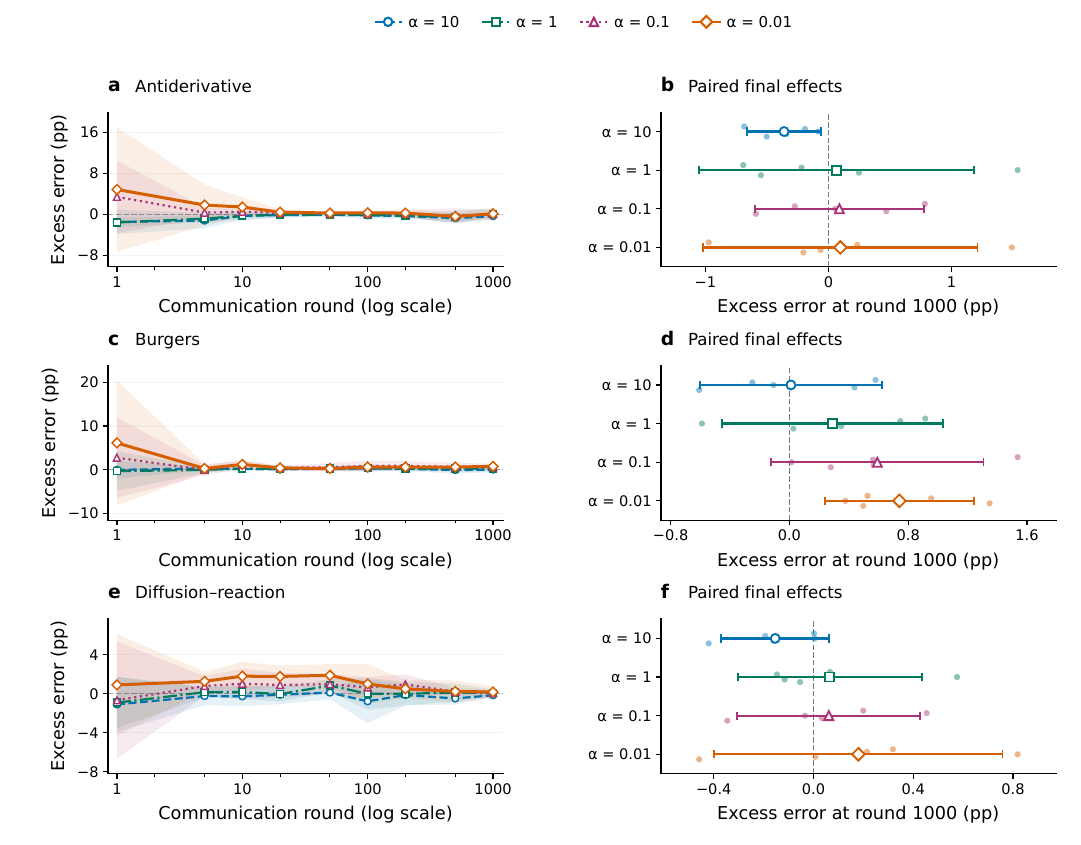}
\caption{Paired prediction differences in controlled tasks. Rows show antiderivative (\textbf{a,b}), Burgers (\textbf{c,d}) and diffusion--reaction (\textbf{e,f}). \textbf{a,c,e}, Mean non-IID minus IID error at recorded rounds, with confidence bands. \textbf{b,d,f}, All five seed-level final effects and the corresponding means and intervals. The zero line marks equal error. Positive differences indicate worse prediction under the non-IID partition. Lines connect recorded checkpoints.}
\end{figure}
\begin{figure}[!htbp]
\centering\includegraphics[width=.97\linewidth,height=.70\textheight,keepaspectratio]{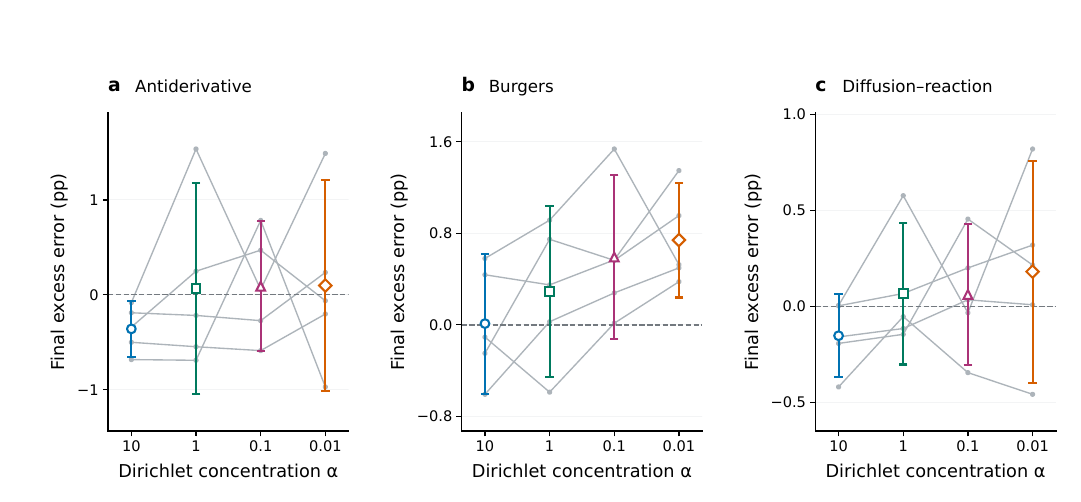}
\caption{Seed-level final effects in controlled tasks. Panels show antiderivative (\textbf{a}), Burgers (\textbf{b}) and diffusion--reaction (\textbf{c}). Each grey line connects the four concentration conditions for one seed. Coloured summary points and error bars show the mean and confidence interval at each concentration. The five seed trajectories each share an IID reference across concentrations.}
\end{figure}
\begin{figure}[!htbp]
\centering\includegraphics[width=.97\linewidth,height=.70\textheight,keepaspectratio]{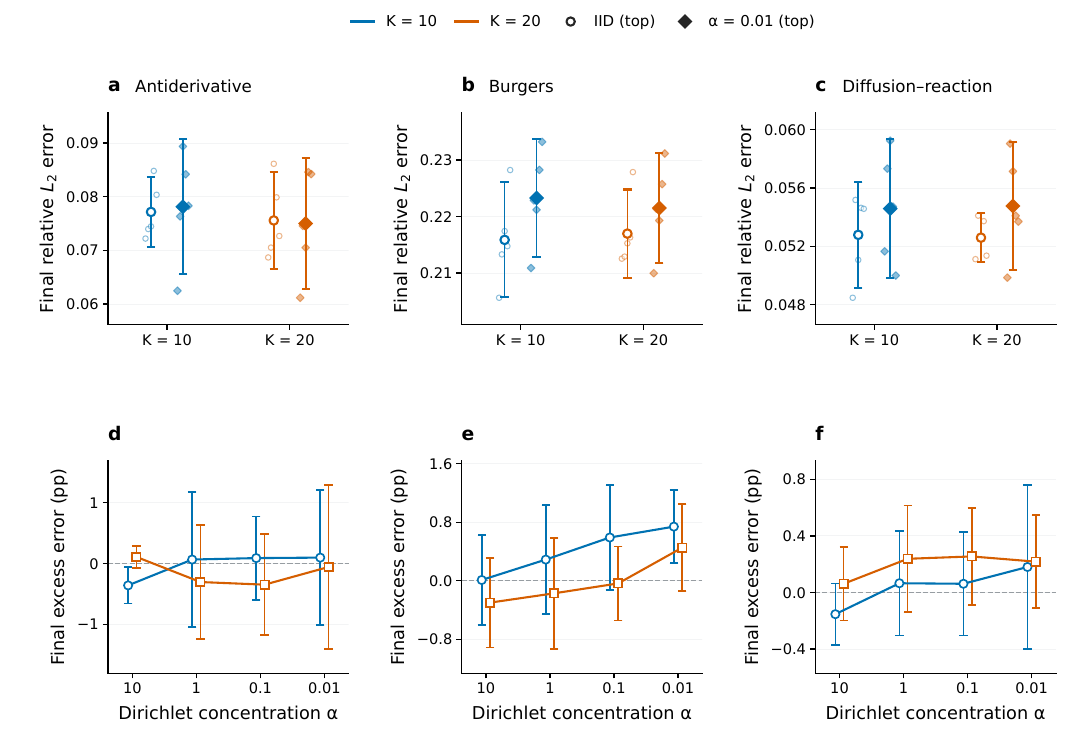}
\caption{Client-count comparison on controlled operators. Columns show antiderivative, Burgers and diffusion--reaction. \textbf{a--c}, Absolute final errors for stratified IID (open circles) and $\alpha=0.01$ (filled diamonds), with individual seed observations and summary intervals. \textbf{d--f}, Paired excess errors for all four concentrations. Each ten- or twenty-client condition is compared with its own matched IID run. All runs have five local steps and 1,000 rounds; full participation gives a larger aggregate update count with twenty clients.}
\end{figure}
\begin{figure}[!htbp]
\centering\includegraphics[width=.95\linewidth,height=.70\textheight,keepaspectratio]{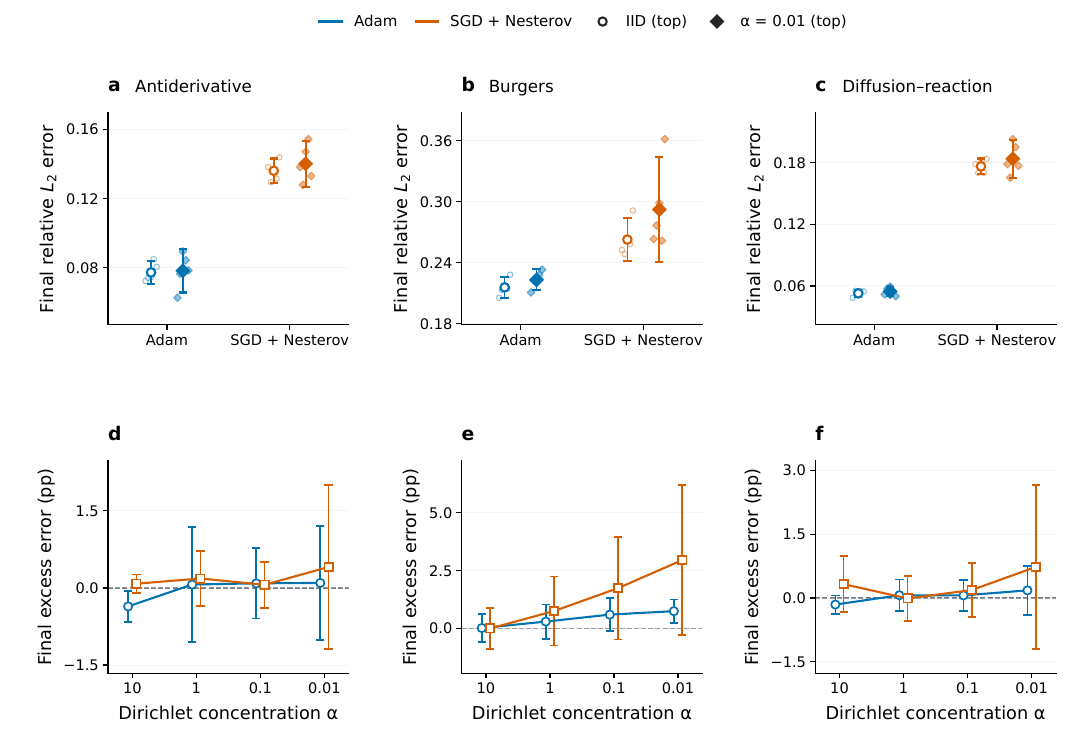}
\caption{Optimizer configurations on controlled operators. Columns show antiderivative, Burgers and diffusion--reaction. \textbf{a--c}, Final absolute errors for IID and $\alpha=0.01$, including all seed observations. \textbf{d--f}, Paired excess errors at four concentrations. Adam uses cosine decay from $10^{-3}$ to $10^{-4}$; Nesterov SGD uses cosine decay from $10^{-2}$ to $10^{-3}$ and momentum 0.9. Both use batch size 16, five local steps and 1,000 rounds. Each non-IID condition is paired with the IID run of the same optimizer configuration.}
\end{figure}
\begin{figure}[!htbp]
\centering\includegraphics[width=.95\linewidth,height=.70\textheight,keepaspectratio]{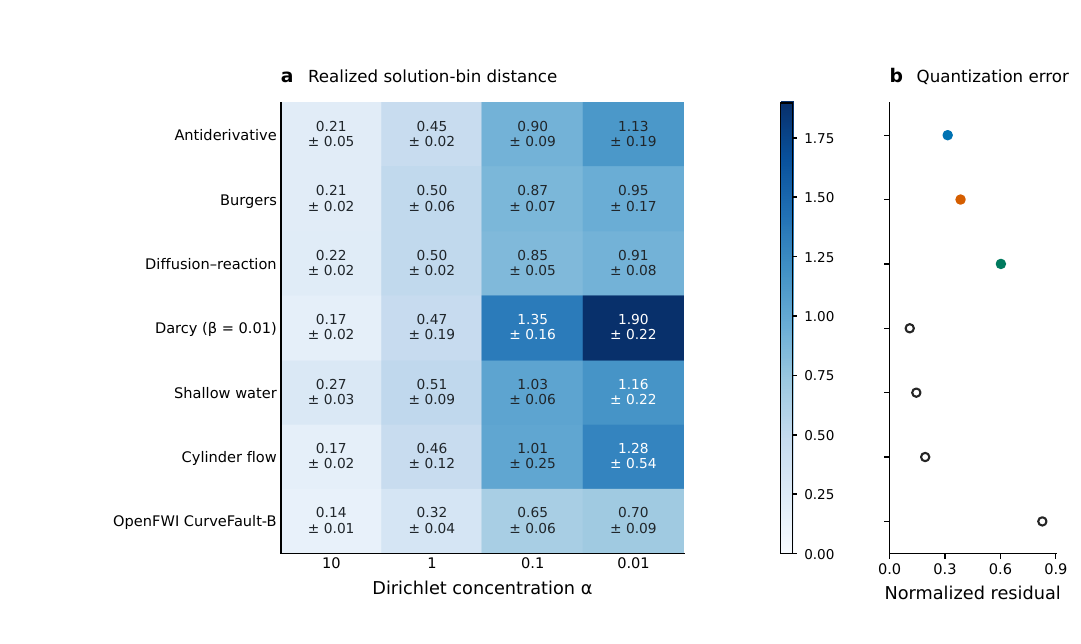}
\caption{Solution distance and discretization fidelity across tasks. \textbf{a}, Mean transport distances and confidence-interval half-widths at four concentrations. \textbf{b}, Mean normalized distance to the assigned response centroid. Controlled-task intervals span codebooks fitted separately for each data seed. Public-task open symbols denote a single fixed split and codebook.}
\end{figure}
\begin{figure}[!htbp]
\centering\includegraphics[width=.97\linewidth,height=.70\textheight,keepaspectratio]{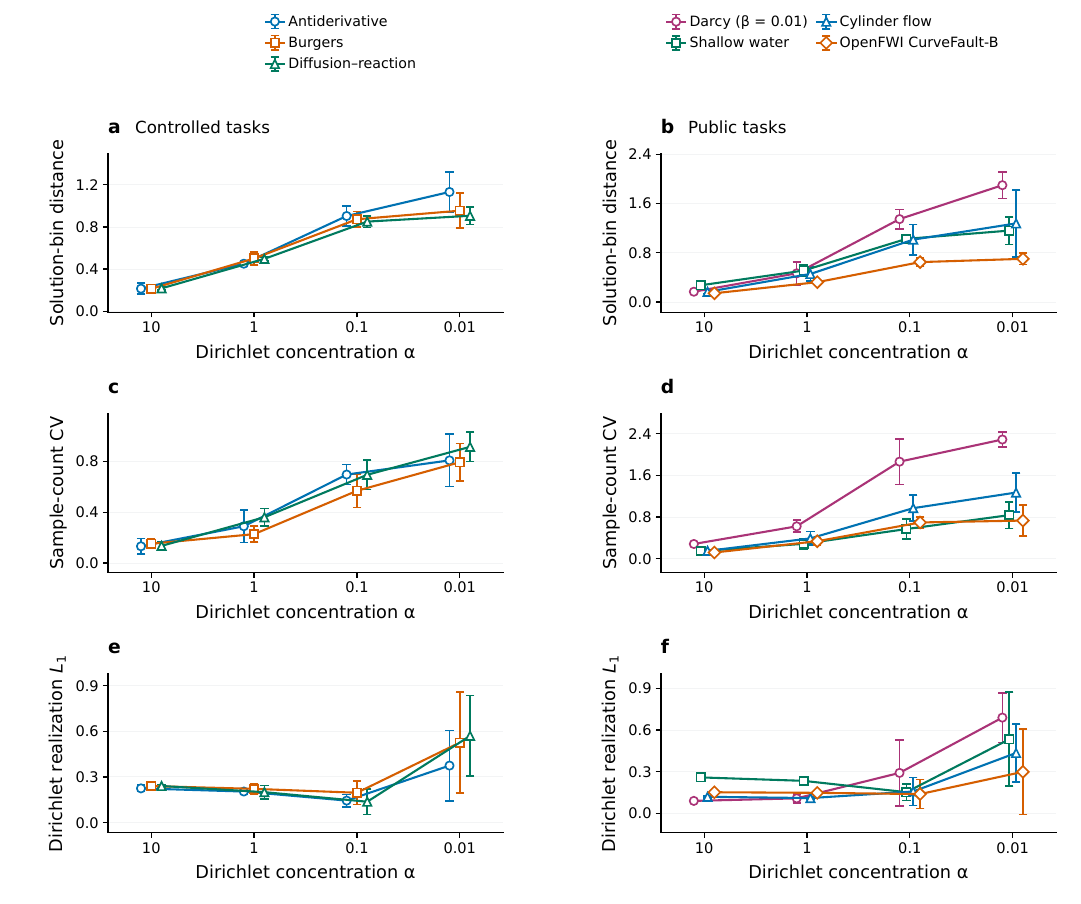}
\caption{Finite-partition diagnostics. Rows show realized solution-bin transport (\textbf{a,b}), the coefficient of variation of client sample counts (\textbf{c,d}), and the mean clientwise $L_1$ difference between realized and target bin proportions (\textbf{e,f}). Columns separate controlled and public tasks. Means and intervals use five seeds. Sampling and minimum-size repair can alter the realized ordering at extreme concentrations.}
\end{figure}
\begin{figure}[!htbp]
\centering\includegraphics[width=.97\linewidth,height=.70\textheight,keepaspectratio]{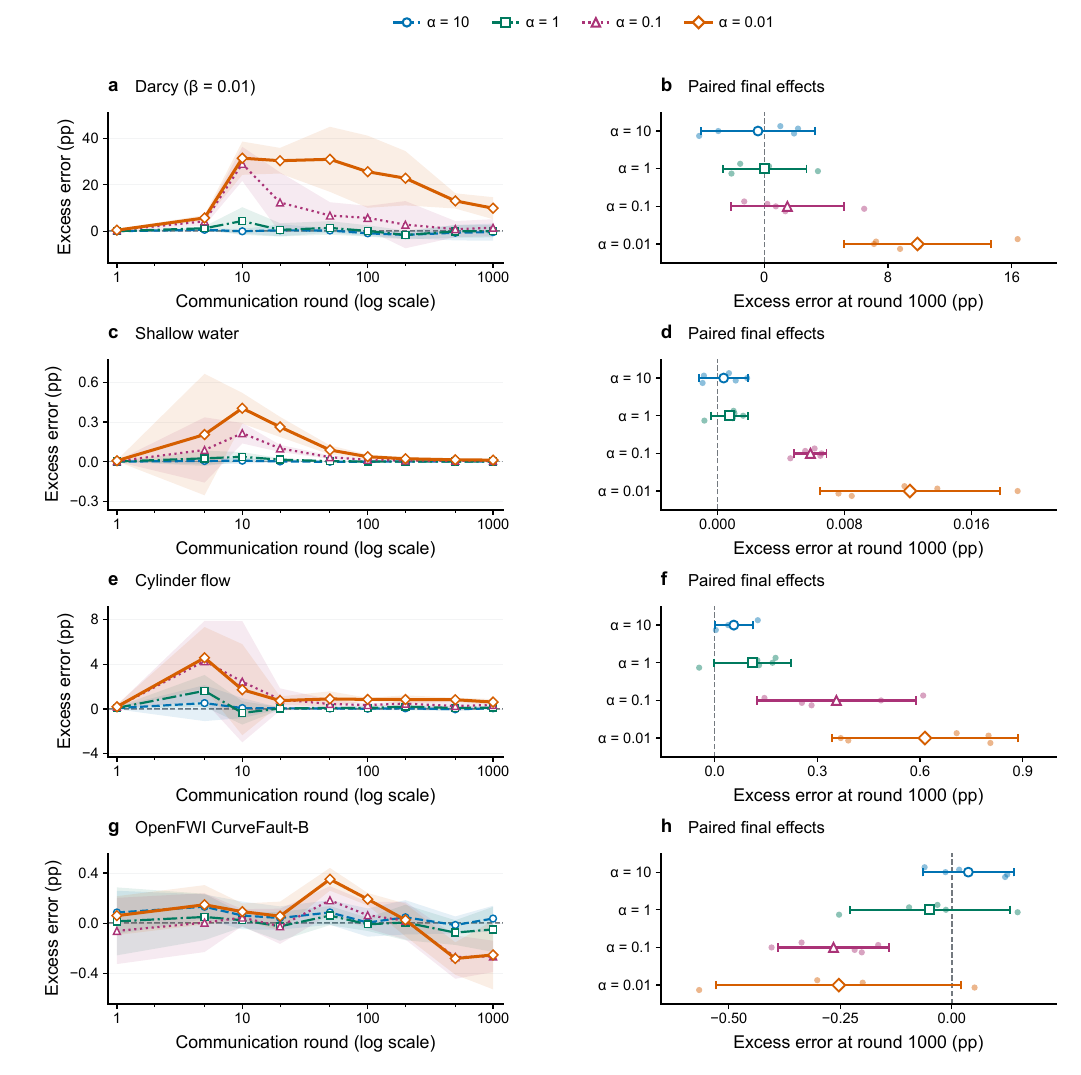}
\caption{Paired prediction differences in public tasks. Rows show Darcy (\textbf{a,b}), shallow water (\textbf{c,d}), cylinder flow (\textbf{e,f}) and CurveFault-B (\textbf{g,h}). \textbf{a,c,e,g}, Recorded-round mean excess errors with confidence bands. \textbf{b,d,f,h}, Final paired effects for every seed, with means and confidence intervals. Each task uses its own error scale. The fixed data split and codebook are shared across the five training and allocation seeds.}
\end{figure}
\begin{figure}[!htbp]
\centering\includegraphics[width=.97\linewidth,height=.70\textheight,keepaspectratio]{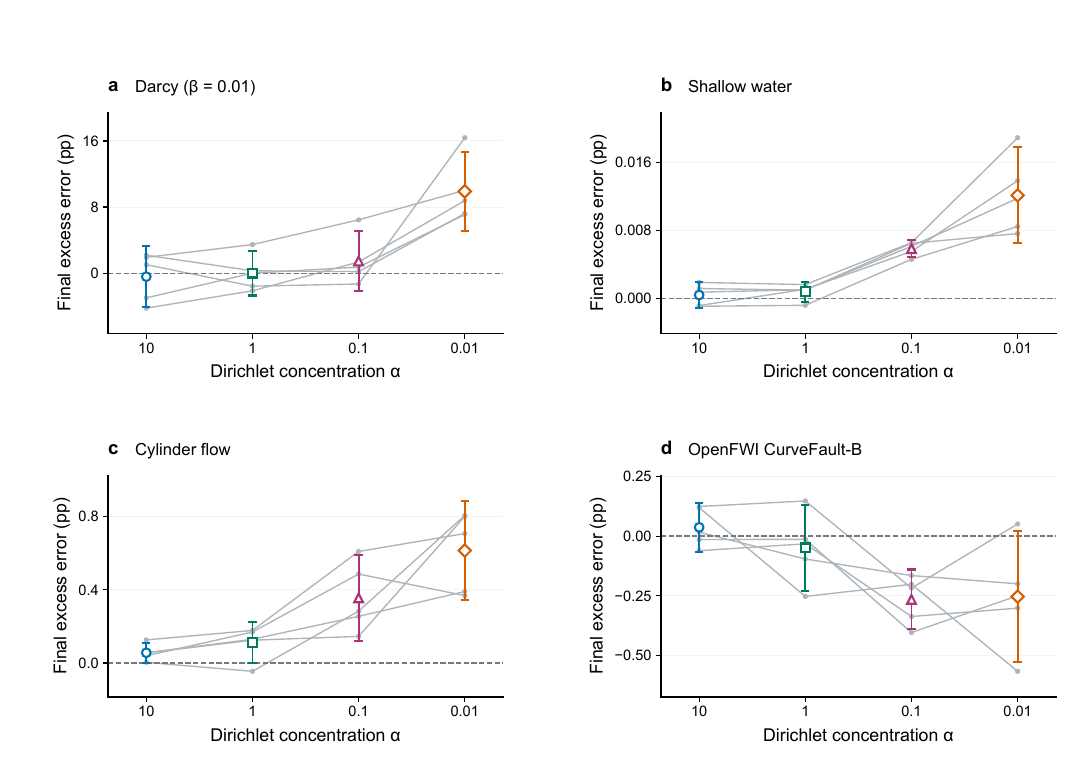}
\caption{Seed-level final effects in public tasks. Panels show Darcy (\textbf{a}), shallow water (\textbf{b}), cylinder flow (\textbf{c}) and CurveFault-B (\textbf{d}). Grey lines connect concentrations within each of the five evaluation seeds; coloured summary points and error bars show means and confidence intervals. Effects are calculated against the same-seed stratified IID model. Darcy has a clear positive shift at the smallest concentration, whereas other tasks show smaller, variable or negative paired effects.}
\end{figure}
\begin{figure}[!htbp]
\centering\includegraphics[width=\linewidth,height=.57\textheight,keepaspectratio]{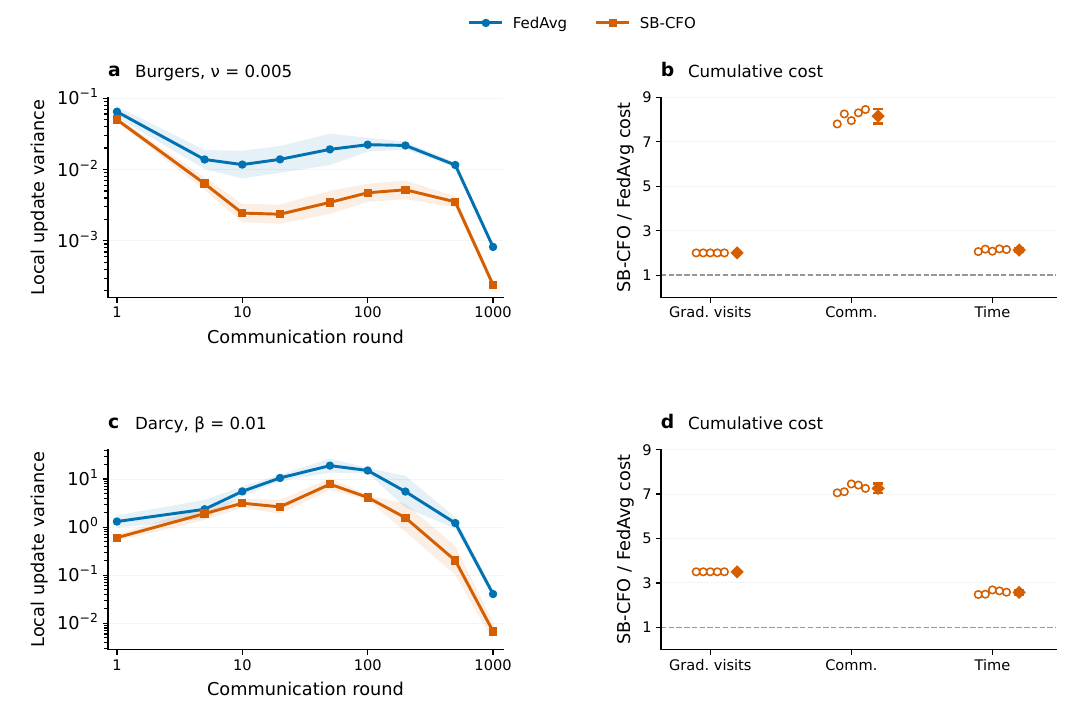}
\caption{Update dispersion and resource cost of SB-CFO at $\alpha=0.01$. Rows show Burgers with $\nu=0.005$ (\textbf{a,b}) and Darcy with $\beta=0.01$ (\textbf{c,d}). \textbf{a,c}, Sample-weighted local update variance at recorded rounds; lines and shading indicate geometric means and back-transformed log-space, pointwise 95\% Student-$t$ intervals across five seeds. Variance describes the preceding round. \textbf{b,d}, Cumulative costs through round 1,000 expressed as within-seed SB-CFO/FedAvg ratios. Grad.\ visits counts optimizer and prototype gradient sample visits, excluding diagnostic-only visits; Comm.\ counts model and prototype tensor bytes; Time is recorded elapsed time including evaluation and diagnostics. Open circles are seeds; diamonds and bars indicate arithmetic means and 95\% intervals of paired ratios. The dashed line marks equal cost. Communication is logical uncompressed payload; timing is measured in the sequential simulator. Full-data prototypes are refreshed every training round.}\label{fig:sbcfocost}
\end{figure}

\end{document}